\documentclass[11pt]{article}

\usepackage{amssymb,amsmath,amsthm,sectsty,url}
\usepackage[letterpaper,hmargin=1.0in,vmargin=1.0in]{geometry}
\usepackage[pdftex,colorlinks,linkcolor=blue,citecolor=blue,filecolor=blue,urlcolor=blue]{hyperref}
\usepackage{cleveref}
\usepackage{color}
\usepackage[boxed]{algorithm}
\usepackage{tikz}
\usepackage{nicefrac}
\usepackage{aliascnt}
\usepackage{algorithmic}

\newtheorem{theorem}{Theorem}[section]
\Crefname{theorem}{Theorem}{Theorems}
\numberwithin{theorem}{section}

\newaliascnt{mylemma}{theorem}
\newtheorem{mylemma}[mylemma]{Lemma}
\aliascntresetthe{mylemma}
\Crefname{mylemma}{Lemma}{Lemmas}

\newaliascnt{myproposition}{theorem}

\aliascntresetthe{myproposition}
\Crefname{myproposition}{Proposition}{Propositions}

\newaliascnt{mycorollary}{theorem}
\newtheorem{mycorollary}[mycorollary]{Corollary}
\aliascntresetthe{mycorollary}
\Crefname{mycorollary}{Corollary}{Corollaries}

\newaliascnt{mydefinition}{theorem}
\newtheorem{mydefinition}[mydefinition]{Definition}
\aliascntresetthe{mydefinition}
\Crefname{mydefinition}{Definition}{Definitions}

\newaliascnt{myremark}{theorem}

\aliascntresetthe{myremark}
\Crefname{myremark}{Remark}{Remarks}

\newaliascnt{myconjecture}{theorem}

\aliascntresetthe{myconjecture}
\Crefname{myconjecture}{Conjecture}{Conjectures}

\newaliascnt{myexample}{theorem}

\aliascntresetthe{myexample}
\Crefname{myexample}{Example}{Examples}

\newaliascnt{definition}{theorem}
\newtheorem{definition}[definition]{Definition}
\aliascntresetthe{definition}
\crefname{definition}{Definition}{Definitions}

\newaliascnt{fact}{theorem}
\newtheorem{fact}[fact]{Fact}
\aliascntresetthe{fact}
\Crefname{fact}{Fact}{Facts}

\newaliascnt{claim}{theorem}
\newtheorem{claim}[claim]{Claim}
\aliascntresetthe{claim}
\Crefname{claim}{Claim}{Claims}

\newaliascnt{question}{theorem}

\aliascntresetthe{question}
\Crefname{question}{Question}{Questions}

\newaliascnt{exercise}{theorem}

\aliascntresetthe{exercise}
\Crefname{exercise}{Exercise}{Exercises}

\newaliascnt{notation}{theorem}

\aliascntresetthe{notation}
\Crefname{notation}{Notation}{Notations}

\newaliascnt{problem}{theorem}

\aliascntresetthe{problem}
\Crefname{problem}{Problem}{Problems}

\newaliascnt{remark}{theorem}
\newtheorem{remark}[remark]{Remark}
\aliascntresetthe{remark}
\Crefname{remark}{Remark}{Remarks}

\newcommand{\norm}[1]{\lVert#1\rVert}

\def\E{\mathbb E}

\newcommand{\R}{\mathbb R}

\newcommand\fsnorm[1]{\ensuremath{\norm{#1}_F^2}} 

\newcommand{\citep}{\cite}
\newcommand{\citet}{\cite}

\title{Generalization Bounds for Data-Driven Numerical Linear Algebra}
\usepackage{times}

\author{%
 Peter Bartlett\thanks{\texttt{peter@berkeley.edu}}\\
 UC Berkeley%
 \and
 Piotr Indyk\thanks{\texttt{indyk@mit.edu}}\\
 MIT%
  \and
 Tal Wagner\thanks{\texttt{tal.wagner@gmail.com}}\\
 Microsoft Research%
}

\begin{document}

\maketitle

\begin{abstract}%
Data-driven algorithms can adapt their internal structure or parameters to inputs from unknown application-specific distributions, by learning from a training sample of inputs. 
Several recent works have applied this approach to problems in numerical linear algebra, obtaining significant empirical gains in performance. 
However, no theoretical explanation for their success was known. 

In this work we prove generalization bounds for those algorithms, within the PAC-learning framework for data-driven algorithm selection proposed by Gupta and Roughgarden (SICOMP 2017). 
Our main results are closely matching upper and lower bounds on the fat shattering dimension of the learning-based low rank approximation algorithm of Indyk et al.~(NeurIPS 2019).
Our techniques are general, and provide generalization bounds for many other recently proposed data-driven algorithms in numerical linear algebra, covering both sketching-based and multigrid-based methods. 
This considerably broadens the class of data-driven algorithms for which a PAC-learning analysis is available.
\end{abstract}

\section{Introduction}
Traditionally, algorithms are formally designed to handle a single unknown input. 
In reality, however, computational problems are often solved on multiple different yet related inputs over time. 
It is therefore natural to tailor the algorithm to the inputs it encounters, and leverage --- or \emph{learn} from --- past inputs, in order to improve performance on future ones. 

This paradigm addresses common scenarios in which we need to make design choices about the algorithm, such as setting parameters or selecting algorithmic components. It advocates making those design choices in an automated way, based on past inputs viewed as a training set, premised to share underlying similarities with future inputs.  
Rather than trying to model or gauge these similarities explicitly, the idea is to let a learning mechanism detect and exploit them implicitly. 
This is often referred to as \emph{self-improving}, \emph{data-driven} or \emph{learning-based} algorithm design, and by now has been applied successfully to a host of computational problems in various domains \citep{ailon2011self,gupta2017pac,balcan2020data,mitzenmacher2020algorithms}.

In particular, this approach has recently become popular in efficient algorithms for computational problems in linear algebra. Matrix computations are very widely used and are often hard to scale, leading to an enormous body of work on developing fast approximate algorithms for them. These algorithms often rely on internal auxiliary matrices, like the sketching matrix in sketch-based methods \citep{woodruff2014sketching} or the prolongation matrix in algebraic multigrid methods \citep{briggs2000multigrid}, and the choice of auxiliary matrix is crucial for good performance. Classically, these auxiliary matrices are chosen either at random (from carefully designed distributions which are oblivious to the input) or handcrafted via elaborate heuristic methods. However, a recent surge of work on learning-based linear algebra shows they can be successfully learned in an automated way from past inputs \citep{indyk2019learning,ailon2020sparse,liu2020learned,luz2020learning,indyk2021few}.


The empirical success of learning-based algorithms has naturally led to seeking theoretical frameworks for reasoning about them. \cite{gupta2017pac} suggested modeling the problem in terms of statistical learning, and initiated a PAC-learning theory for algorithm selection. In their formulation, the inputs are drawn independently from an unknown distribution $\mathcal D$, and the goal is to choose an algorithm from a given class $\mathcal L$, so as to approximately optimize the expected performance on $\mathcal D$. One then proves upper bounds on the \emph{pseudo-dimension} or the \emph{fat shattering dimension} of $\mathcal L$, which are analogs of the VC-dimension for real-valued function classes. By classical VC theory, such bounds imply \emph{generalization}, i.e., that an approximately optimal algorithm for $\mathcal D$ can be chosen from a bounded number of samples (proportional to either notion of dimension). Building on this framework, subsequent works have developed sets of tools for proving pseudo-dimension bounds for learning-based algorithms for various problems, focusing on combinatorial optimization and mechanism design \citep{balcan2017learning,balcan2018learning,balcan2018general,balcan2020refined,balcan2021much,balcan2021sample}.

However, existing techniques do not yield generalization bounds for the learning-based linear algebraic algorithms mentioned above, and so far, no generalization bounds were known for them. 

\paragraph{Our contribution.} 
We develop a new approach for proving PAC-learning generalization bounds for learning-based algorithms, which is applicable to linear algebraic problems. For concreteness, we mostly focus on the learning-based low rank approximation algorithm of \cite{indyk2019learning}, called IVY. 
Operating on input matrices of order $n\times n$, IVY learns an auxiliary sketching matrix specified by $n$ real parameters. 
Our main results is, loosely speaking, that the fat shattering dimension of IVY is $\widetilde\Theta(n)$, which yields a similar bound on its sample complexity.\footnote{Throughout, we use $\widetilde O(f)$ for $O(f\cdot\mathrm{polylog}(f))$.} 
The precise bound we prove also depends on the fatness parameter $\epsilon$, the target low rank $k$, and the row-dimension $m$ and sparsity $s$ of the learned sketching matrix; for now, it is instructive to think of all of them as constants. See \Cref{sec:results} for the full result.

We proceed to show that the tools we develop also lead to pseudo-dimension or fat shattering dimension bounds for various other learning-based linear algebraic algorithms known in the literature, which in addition to IVY include the learned low rank approximation algorithms of \cite{ailon2020sparse}, \cite{liu2020learned} and \cite{indyk2021few}, and the learned regression algorithm of \cite{luz2020learning}. 
These algorithms represent both the sketch-based and the multigrid-based approaches in numerical linear algebra. 

Our work significantly advances the line of research on the statistical foundations of learning-based algorithms, by extending the PAC-learning framework to a prominent and well-studied area of algorithms not previously covered by it, through developing a novel set of techniques.

\section{Background and Main Results}

\subsection{Generalization Framework for Learning-Based Algorithms}\label{sec:guptaroughgarden}
The following is the PAC-learning framework for algorithm selection, due to \cite{gupta2017pac}. 
Let $\mathcal X$ be a domain of inputs for a given computational problem, and let $\mathcal D$ be a distribution over $\mathcal X$. 
Let $\mathcal L$ be a class of algorithms that operate on inputs from $\mathcal X$. 
We identify each algorithm $L\in\mathcal L$ with a function $L:\mathcal X\rightarrow[0,1]$ that maps the result of running $L$ on $x$ to a loss $L(x)$, 
where we shift and normalize the loss to be in $[0,1]$ for convenience. 
Let $L^*\in\mathcal L$ be an optimal algorithm in expectation on $\mathcal D$, i.e., $L^* = \mathrm{argmin}_{L\in\mathcal L}\;\E_{x\sim \mathcal D}[L(x)]$. 

We wish to find an algorithm in $\mathcal L$ that approximately matches the performance of $L^*$. To this end, we draw $\ell$ samples $\tilde X=\{x_1,\ldots,x_\ell\}$ from $\mathcal D$, and apply a \emph{learning procedure} that maps $\tilde X$ to an algorithm $L\in\mathcal L$. 
We say that $\mathcal L$ is \emph{$(\epsilon,\delta)$-learnable} with $\ell$ samples if there is a learning procedure that maps $\tilde X$ to $L\in\mathcal L$ such that
\[ \Pr_{\tilde X}[\E_{x\sim D}[L(x)]\leq\E_{x\sim D}[L^*(x)] + \epsilon]\geq1-\delta . \]

A canonical learning procedure is \emph{Empirical Risk Minimization} (ERM), which maps $\tilde X$ to an algorithm
that minimizes the average loss over the samples, i.e., to $\mathrm{argmin}_{L\in\mathcal L}\;\frac1{\ell}\sum_{i=1}^{\ell}L(x_i)$. 
We say that $\mathcal L$ admits \emph{$(\epsilon,\delta)$-uniform convergence} with $\ell$ samples if
\[ \Pr_{\tilde X}\left[\forall L\in\mathcal L \;\; \left| \frac1{\ell}\sum_{i=1}^{\ell}L(x_i) - \E_{x\sim \mathcal D}[L(x)] \right| \leq \epsilon \right] \geq 1-\delta . \]
It is straightforward that $(\epsilon,\delta)$-uniform convergence implies $(2\epsilon,\delta)$-learnability with ERM with the same number of samples. To bound the number of samples, we have the following notions.

\begin{mydefinition}[pseudo-dimension and fat shattering dimension]\label{def:dim}
Let $X=\{x_1,\ldots,x_N\} \subset \mathcal X$. 
We say that $X$ is \emph{pseudo-shattered} by $\mathcal L$ if there are thresholds $r_1,\ldots,r_N\in\R$ such that,
\[
  \forall I\subset\{1,\ldots,N\} \; \exists L\subset\mathcal L \;\;\;\; \text{s.t.} \;\;\;\; L(x_i)>r_i \Leftrightarrow i\in I .
\]
For $\gamma>0$, we say that $X$ is \emph{$\gamma$-fat shattered} by $\mathcal L$ if there are thresholds $r_1,\ldots,r_N\in\R$ such that,
\[
  \forall I\subset\{1,\ldots,N\} \; \exists L\subset\mathcal L \;\;\;\; \text{s.t.} \;\;\;\; i\in I \Rightarrow L(x_i)>r_i+\gamma   \;\;\;\; \text{and} \;\;\;\; i\notin I \Rightarrow L(x_i)<r_i-\gamma.
\]
The \emph{pseudo-dimension} of $\mathcal L$, denoted $\mathrm{pdim}(\mathcal L)$, is the maximum size of a pseudo-shattered set.
The \emph{$\gamma$-fat shattering dimension} of $\mathcal L$, denoted $\mathrm{fatdim}_\gamma(\mathcal L)$, is the maximum size of a $\gamma$-fat shattered set.
\end{mydefinition}

Note that the pseudo-dimension is an upper bound on the $\gamma$-fat shattering dimension for every $\gamma>0$. 
Classical results in PAC-learning theory show that these quantities govern the number of samples needed for $(\epsilon,\delta)$-uniform convergence, and thus for $(\epsilon,\delta)$-learning with ERM.\footnote{For example, a standard result is that $O(\epsilon^{-2}\cdot(\mathrm{pdim}(\mathcal L)+\log(\delta^{-1})))$ samples suffice for $(\epsilon,\delta)$-uniform convergence, and thus also for $(\epsilon,\delta)$-learning with ERM. The pseudo-dimension yields somewhat tighter upper bounds for learning than the fat shattering dimension, while the latter also yields lower bounds. See, for example, Theorems 19.1, 19.2, 19.5 in \cite{anthony2009neural}.} 
Therefore, a typical goal is to prove upper and lower bounds on them. 


\subsection{Learning-Based Low Rank Approximation (LRA)}
Let $A\in\R^{n\times d}$ be an input matrix, with $n\geq d$. 
By normalization, we assume throughout the paper that $\fsnorm{A}=1$. 
Let $[A]_k$ denote the optimal rank-$k$ approximation of $A$ in the Frobenius norm, i.e.,
\[ [A]_k = \mathrm{argmin}_{A'\in\R^{n\times d} \text{ of rank }k}\fsnorm{A-A'} . \]
It is well-known that $[A]_k$ can be computed using the singular value decomposition (SVD), in time $O(nd^2)$. 
Since this running time does not scale well for large matrices, a very large body of work has been dedicated to developing fast approximate LRA algorithms, which output an matrix $A'$ of rank $k$ that attains error close to $[A]_k$, see surveys by \cite{mahoney2011randomized,woodruff2014sketching,martinsson2020randomized}. 

\paragraph{SCW.} 
The SCW algorithm for LRA is due to \cite{clarkson2013low}, building on \cite{sarlos2006improved,clarkson2009numerical}. It uses an auxiliary \emph{sketching matrix} $S\in\R^{m\times n}$, where its row-dimension $m$ is called the \emph{sketching dimension}, and is chosen to be slightly larger than $k$ and much smaller than $n$. 
The algorithm is specified in \Cref{alg:scw}. 

In \cite{clarkson2013low}, $S$ is chosen at random from an data-oblivious distribution as follows. 
Let $s\in\{1,\ldots,m\}$ be a \emph{sparsity} parameter. One chooses $s$ uniformly random entries in each column of $S$, and chooses the value of each of them uniformly at random from $\{1,-1\}$. The rest of the entries in $S$ are zero. 
\cite{clarkson2013low} show that given $\epsilon>0$, by setting the sketching dimension to $m=\widetilde O(k^2/\epsilon^2)$, even with sparsity $s=1$, SCW returns $A'$ of rank $k$ that satisfies $\fsnorm{A-A'}\leq(1+\epsilon)\fsnorm{A-[A]_k}$ with high probability (over the random choice of $S$), while running in time nearly linear in the size of $A$.

\newcommand{\INDENT}{\hspace{1em}}
\begin{algorithm}[!t]
\caption{The SCW low rank approximation algorithm}
\label{alg:scw}
\smallskip
\textbf{Input:} $A\in\R^{n\times d}$, target rank $k$, auxiliary matrix $S\in\R^{m\times n}$. 
\textbf{Output:} $A'\in\R^{n\times d}$ of rank $k$.
\smallskip{\hrule height.2pt}\smallskip
\begin{algorithmic}[1]
   \STATE Compute the product $SA$.
   \STATE If $SA$ is a zero matrix, return the zero matrix of order $n\times d$.
   \STATE Compute the SVD $U,\Sigma,V^T$ of $SA$.
   \STATE Compute the product $AV$.
   \STATE Compute and return $[AV]_kV^T$.
\end{algorithmic}
\end{algorithm}

\paragraph{IVY.}
The IVY algorithm due to \cite{indyk2019learning} is a learning-based variant of SCW.
The sparsity pattern of $S$ is chosen similarly to SCW ($s$ non-zero entries per column, chosen uniformly at random) and remains fixed. However, the values of the non-zero entries in $S$ are now trainable parameters, learned from a training set of input matrices. 
In particular, $S$ is chosen by minimizing the empirical loss $\sum_{A\in\mathcal A_{train}}\fsnorm{A-\mathrm{SCW}_k(S,A)}$, where $\mathrm{SCW}_k(S,A)$ denotes the output matrix of \Cref{alg:scw}, using stochastic gradient descent (SGD) on a training set $\mathcal A_{train}\subset\R^{n\times d}$.

To cast IVY in the statistical learning framework from \Cref{sec:guptaroughgarden}, let $\mathcal A$ denote the set of possible input matrices (i.e., all $A\in\R^{n\times d}$ with $\fsnorm{A}=1$), and let $\mathcal S$ denote the set of possible sketching matrices (i.e., all $S\in\R^{m\times n}$ with the fixed sparsity pattern specified above). 
Every $S\in\mathcal S$ gives rise to an LRA algorithm, whose output rank-$k$ approximation of $A$ is $A'=\mathrm{SCW}_k(S,A)$, and whose associated loss function $L^{\mathrm{SCW}}_k(S,\cdot):\mathcal A \rightarrow[0,1]$ is given by\footnote{It can be shown that the loss never exceeds $\fsnorm{A}$ for any $A$ and $S$, and recalling that we assume that all input matrices are normalized so that $\fsnorm{A}=1$, the loss is contained in $[0,1]$.}
\[ L^{\mathrm{SCW}}_k(S,A)=\fsnorm{A-\mathrm{SCW}_k(S,A)} . \]
Given samples from an unknown distribution $\mathcal D$ over $\mathcal A$, the learning problem is to choose $S\in\mathcal S$ which has approximately optimal loss in expectation over $\mathcal D$. 
Our objective is to prove generalization bounds for learning a sketching matrix $S\in \mathcal S$, or equivalently, for learning the class of loss functions $\mathcal L_{\mathrm{IVY}}=\{L^\mathrm{SCW}_k(S,\cdot)\}_{S\in\mathcal S}$.


\subsection{Our Main Results}\label{sec:results}
We prove closely matching upper and lower bounds on the fat shattering dimension of IVY. 

\begin{theorem}\label{thm:main}
For every $\epsilon>0$ smaller than a sufficiently small constant, the $\epsilon$-fat shattering dimension of IVY, $\mathrm{fatdim}_\epsilon(\mathcal L_{\mathrm{IVY}})$, satisfies
\[ \mathrm{fatdim}_\epsilon(\mathcal L_{\mathrm{IVY}}) \leq O(ns\cdot (m + k\log (d/k) + \log(1/\epsilon))) . \]
Furthermore, if $\epsilon<1/(2\sqrt{k})$, then $\mathrm{fatdim}_\epsilon(\mathcal L_{\mathrm{IVY}}) \geq \Omega(ns)$.
\end{theorem}
These results translate to sample complexity bounds on its uniform convergence and learnability:

\begin{theorem}\label{thm:samples}
Let $\epsilon,\delta>0$ be smaller than sufficiently small constants. 
The number of samples needed for $(\epsilon,\delta)$-uniform convergence for IVY, and thus for $(\epsilon,\delta)$-learning the sketching matrix of IVY with ERM, is
\[ O(\epsilon^{-2}\cdot (ns\cdot (m + k\log (d/k) + \log(1/\epsilon)) + \log(1/\delta))) . \]
Furthermore, if $\epsilon\leq1/(256\sqrt{k})$, then $(\epsilon,\epsilon)$-uniform convergence for IVY requires $\Omega(\epsilon^{-2}\cdot ns/k)$ samples, and if $\epsilon<1/(2\sqrt{k})$, then $(\epsilon,\delta)$-learning IVY with any learning procedure requires $\Omega(\epsilon^{-1}+ns)$ samples.
\end{theorem}

For the sake of intuition, let us comment on typical settings for the various sizes in these results. 
The input matrix is of order $n\times d$, where we think of $n,d$ as being arbitrarily large.\footnote{Recall we assume that $d\leq n$ by convention. For all conceptual purposes it suffices to consider the square case $n=d$.} 
The target rank $k$ can be thought of as a small integer constant, and the sketching dimension $m$ as only slightly larger (the larger $m$ is, the slower but more accurate the LRA algorithm would be). Both are generally independent of $n,d$. Concretely, the empirical evaluations in \cite{indyk2019learning,indyk2019sample,indyk2021few,ailon2020sparse,liu2020learned} use $k$ up to $40$, and $m$ up to $4k$, for matrices with thousands of rows and columns. The sparsity $s$ is often chosen to be the smallest possible, $s=1$, while some have found it beneficial to make it slightly larger, up to $8$ \citep{tropp2019streaming,martinsson2020randomized}. 
The upshot is that the upper and lower bounds in \Cref{thm:main} are essentially matching up to a factor of $\log(d/\epsilon)$. Furthermore, both are essentially proportional to $ns$, which is the number of non-zero entries in the sketching matrix, i.e., the number of trainable parameters that IVY learns.

\paragraph{Other learning-based algorithms.} While we focus on IVY for concreteness, our techniques also yield bounds for many other learning-based linear algebraic algorithms. See \Cref{sec:other}.

\begin{remark}[on computational efficiency]
Like most prior work on PAC-learning, our results focus on the sample complexity of the learning problem, rather than the computational efficiency of learning (see, e.g., Remark 3.5 in~\cite{gupta2017pac}). 
It is currently unknown whether there exist efficient ERM learners for the data-driven algorithms considered in this work. In particular, while in practice IVY uses SGD to minimize the empirial loss, it is not known to provably converge to a global minimum. The computational efficiency of ERM for backpropagation-based algorithms (like IVY) remains a challenging open problem.
\end{remark}

\begin{remark}[on the precision of real number representation]
We prove \Cref{thm:main} without any restriction on the numerical precision of real numbers in the computational model, that is, even when they may have unbounded precision. 
If one considers only bounded precision models, say where real numbers are represented by $b$-bit machine words, then the total number of possible sketching matrices is $2^{nsb}$, leading trivially to a bound of $O(nsb)$ on the pseudo-dimension of IVY. On the other hand, the lower bound in \Cref{thm:main} holds even with $1$-bit machine words. 
In summary, both our upper and lower bounds are proven in the most general computational model. 
\end{remark}

\subsection{Technical Overview}
Our starting point is a general framework of \cite{goldberg1995bounding} for bounding the VC-dimension or the pseudo-dimension. They showed that if the functions in a class (which in our case are the losses of candidate algorithms) can be computed by a certain type of algorithm, which we call a \emph{GJ algorithm}, then the pseudo-dimension of the function class can be bounded in terms of the running time of that algorithm. We employ a simple yet crucial refinement of this framework, relying on more refined complexity measures of GJ algorithms than the running time. 


Under this framework, ideally we would have liked to compute the SCW loss with a GJ algorithm which is efficient in these refined complexity measures, thus obtaining a bound on the pseudo-dimension of IVY. 
Unfortunately, it is not clear how to compute the SCW loss at all with a GJ algorithm, since the GJ framework only allows a narrow set of operations. 
Nonetheless, we show it can be \emph{approximated} by a GJ algorithm, by relying on recent advances in numerical linear algebra, and specifically on ``gap-independent'' analyses of power method iterations for LRA, that do not depend on numerical properties of the input matrix (like eigenvalue gaps)  \citep{rokhlin2010randomized,halko2011finding,boutsidis2014near,woodruff2014sketching,witten2015randomized,musco2015randomized}.
We thus obtain a pseudo-dimension bound for the approximate loss, which translates to a fat shattering dimension bound for the true SCW loss, yielding generalization results for IVY, as well as for other learning-based numerical linear algebra algorithms. 

\paragraph{Proof roadmap.}
\Cref{sec:GJ} presents the \cite{goldberg1995bounding} framework. 
In \Cref{sec:mk1}, as a warm-up, we prove tight bounds for IVY in the simple case $m=k=1$. 
\Cref{sec:ub} contains the main proof of this paper, establishing the upper bound from \Cref{thm:main}. 
The lower bound is proven in \Cref{app:lb}, together completing the proof of \Cref{thm:main}. 
\Cref{thm:samples} follows from \Cref{thm:main} using mostly standard arguments, given in \Cref{app:samples}. 

\subsection{Related Work}

\cite{indyk2019learning} gave a ``safeguarding'' technique for IVY, showing it can be easily modified to guarantee that the learned algorithm never performs worse than the oblivious SCW algorithm on future matrices. The modification is simply to concatenate an oblivious sketching matrix vertically to the learned sketching matrix. While this result guarantees that learning does not hurt, it does not show that learning can help, and provides no generalization guarantees. 

\cite{indyk2021few} prove ``consistency'' results for their learning-based algorithms, showing that the learned sketching matrix performs well on the training input matrices from which it was learned. These results have no bearing on future matrices, and again provide no generalization guarantees. 

\cite{balcan2021much} recently gave a general technique for proving generalization bounds for learning-based algorithm in the statistical learning framework of \Cref{sec:guptaroughgarden}, based on piecewise decompositions of dual function classes. While there are some formal connections between their techniques and ours, their approach does not yield useful bounds for the linear algebraic algorithms that we study, since these algorithms do not exhibit a sufficiently simple piecewise dual structure as the technique of \cite{balcan2021much} requires. We discuss this in more detail in \Cref{app:related}.

\section{The Goldberg-Jerrum Framework}\label{sec:GJ}
Our upper bounds are based on a general framework due to \cite{goldberg1995bounding} for bounding the VC-dimension. We instantiate a refined version of it, which still follows immediately from their proofs. For completeness and self-containedness, \Cref{app:gj} reproduces the proof of the variant we state below.

\begin{mydefinition}\label{def:GJ}
A \textbf{GJ algorithm} $\Gamma$ operates on real-valued inputs, and can perform two types of operations:
\begin{itemize}
    \item Arithmetic operations of the form $v''=v\odot v'$, where $\odot\in\{+,-,\times,\div\}$.
    \item Conditional statements of the form ``if $v\geq0$ ... else ...''.
\end{itemize}
In both cases, $v,v'$ are either inputs or values previously computed by the algorithm.
\end{mydefinition}

Every intermediate value computed by $\Gamma$ is a multivariate rational function (i.e., the ratio of two polynomials) of its inputs. The degree of a rational function is the maximum of the degrees of the two polynomials in its numerator and denominator, when written as a reduced fraction. 
We now define two complexity measures of GJ algorithms.

\begin{mydefinition}\label{def:gjcomplexity}
The \textbf{degree} of a GJ algorithm is the maximum degree of any rational function it computes of the inputs. 
The \textbf{predicate complexity} of a GJ algorithm is the number of distinct rational functions that appear in its conditional statements. 
\end{mydefinition}

\begin{theorem}\label{thm:GJ}
Using the notation of \Cref{sec:guptaroughgarden}, suppose that each algorithm $L\in\mathcal L$ is specified by $n$ real parameters.
Suppose that for every $x\in \mathcal X$ and $r\in \R$ there is a GJ algorithm $\Gamma_{x,r}$ that given $L\in\mathcal L$, returns ``true'' if $L(x)>r$ and ``false'' otherwise. 
Suppose $\Gamma_{x,r}$ has degree $\Delta$ and predicate complexity $p$.
Then, the pseudo-dimension of $\mathcal L$ is at most $O(n\log(\Delta p))$.
\end{theorem}

Let us comment on the relation between this statement and the original theorem from \cite{goldberg1995bounding}. 
Their statement is that if $\Gamma_{x,r}$ has running time $t$ then the pseudo-dimension of $\mathcal L$ is $O(nt)$. They prove it by implcitly proving \Cref{thm:GJ}, as it is not hard to see that if the running time is $t$ then both the degree and the predicate complexity are at most $2^t$. 

To illustrate why the refinement stated in \Cref{thm:GJ} could be useful, let us first consider the degree. Consider computing $q$ power method iterations for a matrix $M\in\R^{n\times n}$, i.e., computing $M^q\pi$ with some initial vector $\pi\in\R^n$. The running time is $t=O(n^2q)$, so the pseudo-dimension upper bound we get based on the running time alone is $O(nt)=O(n^3q)$. However, the degree of every entry in $M^q\pi$ (when considered as a rational function of the entries of $M$ and $\pi$) is just $q+1$, and hence \Cref{thm:GJ} yields the better upper bound $O(n\log q)$.

As for the predicate complexity, consider for example a GJ algorithm for choosing the minimum of $r$ numbers (this operation will be useful for us in derandomizing the power method). The running time is $t=O(r)$, and this is tight since GJ algorithms adhere to the comparison model, so the pseudo-dimension upper bound we get based on the running time alone is $O(nt)=O(nr)$. However, the predicate complexity is ${r\choose 2}$, and the degree is just $1$, since choosing the minimum among $v_1,\ldots,v_r\in\R$ involves only the polynomials $v_i-v_j$ for every $i<j$. Hence, \Cref{thm:GJ} yields the better upper bound $O(n\log r)$. The same reasoning applies to sorting $r$ numbers (see \Cref{app:knapsack} for an example that uses this operation), and to other useful subroutines.

\begin{remark}[oracle access to optimal solution]\label{rem:gjoracle}
It is also worth observing that $\Gamma_{x,r}$ has free access to all information about $x$ --- and in particular, to the optimal solution of the computational problem addressed by $\mathcal L$ with $x$ as the input--- without any computational cost. For example, in LRA, $\Gamma_{x,r}$ has free access to the exact SVD factorization of the input matrix $A$. (Note that the difficulty for $\Gamma_{x,r}$ lies in computing the losses of other than optimal solutions, namely the SCW loss induced by the given sketching matrix $S$). Similarly, in a regression problem $\min_y\norm{Ay-b}$, the input $x$ is the pair $A,b$, and $\Gamma_{x,r}$ has free access to the optimal solution $y^*=\mathrm{argmin}_y\norm{Ay-b}$. 
These observations are useful in proving bounds for some of the algorithms we consider in \Cref{sec:other}.
\end{remark}


\section{Warm-up: The Case $m=k=1$}\label{sec:mk1}
As a warm-up, we consider the simple case where both the target rank $k$ and the sketching dimension is $m$ are $1$. We show an upper bound of $O(n)$ on the pseudo-dimension and a lower bound of $\Omega(n)$ on the $\epsilon$-fat shattering dimension for every $\epsilon\in(0,0.5)$ (this immediately implies that both are $\Theta(n)$).
This case is particularly simple because the SCW loss admits a closed-form expression, given next (the proof appears in \Cref{app:mk1}). Note that the sketching matrix in this case is just a vector, denoted $w\in\R^n$ for the rest of this section, and it is specified by $n$ real parameters.

\begin{fact}\label{fct:scw1dim}
The loss of $SCW$ with input matrix $A\in\R^{n\times d}$, sketching vector $w\in\R^n$ and target rank $k=1$ is 
$\fsnorm{A-\frac{1}{\norm{A^Tw}^2}AA^Tww^TA}$.
\end{fact}

\begin{theorem}
The pseudo-dimension of IVY in the $m=k=1$ case is $\Theta(n)$, and the $\epsilon$-fat shattering dimension is $\Theta(n)$ for every $\epsilon\in(0,0.5)$
\end{theorem}
\begin{proof}
Since the pseudo-dimension is an upper bound on the fat shattering dimension, it suffices to prove the upper bound for the former and the lower bound for the latter. 
For a fixed matrix $A$ and threshold $r\in\R$, a GJ algorithm $\Gamma_{A,r}$ can evaluate the loss formula from 
\Cref{fct:scw1dim} on a given $w$ with degree $O(1)$, and compare the loss to $r$ with predicate complexity $1$. Therefore, by \Cref{thm:GJ}, the pseudo-dimension is $O(n)$. 

For the lower bound, we first argue that if $A$ is a rank-$1$ matrix, then SCW with sketching vector $w$ attains zero loss if and only if $w^TA\neq0$. Indeed, by \Cref{fct:scw1dim}, if $w^TA=0$ then the loss is $\fsnorm{A}$, which is non-zero for a rank-$1$ matrix. If $w^TA\neq 0$, we write $A$ in SVD form $A=\sigma uv^T$ with $\sigma\in\R$, $u\in\R^n$, $v\in\R^{d}$, and $\norm{u}=\norm{v}=1$, and by plugging this into \Cref{fct:scw1dim}, the loss is $0$. 

Now, consider the set of matrices $A_1,\ldots,A_n\in\R^{n\times d}$, where $A_i$ is all zeros except for a single $1$-entry in row $i$, column $1$. For every $I\subset\{1,\ldots,n\}$, let $w_I\in\R^n$ be its indicator vector. 
By the above, SCW with sketching vector $w_I$ attains zero loss on the matrices $\{A_i:i\in I\}$, and loss $\fsnorm{A}=1$ on the rest. Therefore this set of $n$ matrices is $\epsilon$-fat shattered for every $\epsilon\in(0,0.5)$.
\end{proof}

\section{Proof of the Upper Bound}\label{sec:ub}
In this section we prove the upper bound in \Cref{thm:main} on the fat shattering dimension of IVY. 
We start by stating a formula for the output rank-$k$ matrix of SCW (\Cref{alg:scw}). 
For a matrix $M$, we use $M^\dagger$ to denote its Moore-Penrose pseudo-inverse, and recall that we use $[M]_k$ to denote its best rank-$k$ approximation. 
All of the proofs omitted from this section are collected in \Cref{app:proofs}.

\begin{mylemma}\label{lmm:scwalt}
The output matrix of the SCW algorithm equals $[A(SA)^\dagger(SA)]_k$.
\end{mylemma}

\subsection{Computing Projection Matrices}
Next, we give a GJ algorithm for computing orthogonal projection matrices. 
Recall that the orthogonal projection matrix on the row-space of a matrix $Z$ is given by $Z^\dagger Z$.

\begin{mylemma}\label{lmm:proj}
There is a GJ algorithm that given an input matrix $Z$ with $k$ rows, computes $Z^\dagger Z$. The algorithm has degree $O(k)$ and predicate complexity at most $2^k$. 
Furthermore, if $Z$ is promised to have full rank $k$, then the predicate complexity is zero.
\end{mylemma}
\begin{proof}
We start by proving the ``furthermore'' part, supposing $Z$ has full rank $k$.
Since $Z$ has full rank $k$, by a known identity for the pseudo-inverse we can write $Z^\dagger Z=Z^T(ZZ^T)^{-1}Z$.
We use the matrix inversion algorithm from \cite{csanky1976fast} (similar/identical algorithms have appeared in other places as well), to invert an invertible $k\times k$ matrix $M$. Define the following matrices,
\[ B_1 = I, \;\;\;\; B_i = MB_{i-1} - \frac{\mathrm{tr}(MB_{i-1})}{k-1}\cdot I . \]
The inverse is then given by,
\[ M^{-1} = \frac{k}{\mathrm{tr}(MB_k)}\cdot B_k . \]
Note that each entry of $B_i$ is a polynomial of degree $i-1$ in the entries of $M$, and therefore each entry of $M^{-1}$ is a rational function of degree $k-1$ in the entries of $M$. In our case $M=ZZ^T$, and the desired output is $Z^T(ZZ^T)^{-1}Z$, and therefore it has degree $O(k)$. This finishes the proof of the ``furthermore'' part.

We proceed to showing the lemma without the full rank assumption. Suppose $M$ has rank $r\leq k$ which could be strictly smaller than $k$ (and $r$ need not be known to the algorithm). If we find a matrix $Y\in\R^{r\times d}$ whose rows form a basis for the row-space of $M$, then we could write the desired projection matrix as $Z^\dagger Z=Y^\dagger Y=Y^T(YY^T)^{-1}Y$, and compute it with a GJ-algorithm of degree $O(r)\leq O(k)$ in the entries of $Y$, as above.

So, it remains to compute such a matrix $Y$. To this end we use another result from \cite{csanky1976fast}: if $M$ is a $k\times k$ matrix with characteristic polynomial $f_M(\lambda)=\mathrm{det}(\lambda I-M)=\sum_{i=0}^kc_i\lambda^{k-i}$, then $c_i = -\frac{1}{i}\mathrm{tr}(MB_i)$,
with $B_i$ defined as above. In particular, the free coefficient $c_k$ can be computed by a GJ algorithm of degree $O(k)$. This yields a GJ algorithm for determining whether a $k\times k$ matrix has full rank, since this holds if and only if $c_k\neq0$.

Now, to compute $Y$ from $Z$, we can simply go over the rows of $Z$ one by one, tentatively add each row to our (partial) $Y$, and keep or remove it depending on whether $Y$ still has full row rank. To check this, we compute $YY^T$ (which is a square matrix of order at most $r\leq k$) and check whether it has full rank as above. Overall, $Y$ and therefore the output $Y^T(YY^T)^{-1}Y$ are computed using a GJ algorithm with degree $O(k)$. Choosing $Y$ involves conditional statements with up to $2^k$ rational functions, which are the free coefficients of the characteristic polynomials of every subset of the $k$ rows of $Z$, so the predicate complexity is at most $2^k$.
\end{proof}

We remark that this lemma already yields an upper bound on the pseudo-dimension of IVY in the case $m=k$. The full regime $m\geq k$ is more involved and will occupy the rest of this section. 

\begin{mycorollary}\label{cor:ubmk}
The pseudo-dimension of IVY with $m=k$ is $O(nsk)$.
\end{mycorollary}
\begin{proof}
$S$ has $m=k$ rows, hence $SA$ has rank at most $k$, hence $[A(SA)^\dagger (SA)]_k=A(SA)^\dagger (SA)$. By \Cref{lmm:scwalt}, the SCW loss is $\fsnorm{A-A(SA)^\dagger (SA)}$. 
By \Cref{lmm:proj}, it can be computed with a GJ algorithm of degree $O(k)$ and predicate complexity $2^k$. Recalling that the learned sketching matrix in IVY is specified by $ns$ real parameters, the corollary follows from \Cref{thm:GJ}.
\end{proof}

\subsection{Derandomized Power Method Iterations}
For the rest of this section we denote $B = A(SA)^\dagger (SA)$ for brevity. Note that $B$ is of order $n\times d$. By \Cref{lmm:scwalt} the SCW loss equals $\fsnorm{A-[B]_k}$, while \Cref{lmm:proj} shows how to compute $B$ with a GJ algorithm. Our next goal is to approximately compute $[B]_k$ with a GJ algorithm. 

To this end we use a result on iterative methods for LRA, due to \cite{musco2015randomized}, building on ideas from \cite{rokhlin2010randomized,halko2011finding,witten2015randomized,boutsidis2014near,woodruff2014sketching}. 
It states that given $B$, an initial crude approximation for $[B]_k$ can be refined into a good approximation with a small number of powering iterations. 
The next theorem is somewhat implicit in \cite{musco2015randomized}; see \Cref{app:muscomuscopowermethod} for details.

\begin{theorem}[\cite{musco2015randomized}]\label{thm:muscomusco}
Suppose we have a matrix $P\in\R^{d\times k}$ that satisfies
\begin{equation}\label{eq:muscocond}
    \fsnorm{B-(BP)(BP)^\dagger B} \leq O(kd)\cdot\fsnorm{B-[B]_k}.
\end{equation}
Then, $Z=(BB^T)^qBP$ with $q=O(\epsilon^{-1}\log(d/\epsilon))$ satisfies $\fsnorm{B - ZZ^\dagger B} \leq (1+\epsilon)\cdot\fsnorm{B-[B]_k}$. 
\end{theorem}

Normally, $P$ is chosen at random as a matrix of independent gaussians, which satisfies \Cref{eq:muscocond} with high probability. Since GJ algorithms are deterministic, we derandomize this approach using subsets of the standard basis.

\begin{mylemma}\label{lmm:initialp}
Suppose $k<d$. 
For every $B\in\R^{n\times d}$, there is a subset of size $k$ of the standard basis in $\R^d$, such that if we organize its elements into a matrix $P\in\R^{d\times k}$, it satisfies \Cref{eq:muscocond}.
\end{mylemma}

\subsection{The Proxy Loss}\label{sec:proxy}
Let $L^\mathrm{SCW}_k(S,A)$ denote the loss of SCW with input matrix $A$, sketching matrix $S$ and target rank $k$. 
By \Cref{lmm:scwalt}, $L^\mathrm{SCW}_k(S,A) = \fsnorm{A-A(SA)^\dagger (SA)}$. 
We now approximate this quantity with a GJ algorithm. Formally, given $\epsilon>0$, we define a proxy loss $\hat L_{k,\epsilon}(S,A)$ as the output of the following GJ algorithm, which operates on an input $S$ with a fixed $A$. 
Recall that $d$ is the column dimension of $A$.

\begin{enumerate}
    \item Compute $B=A(SA)^\dagger(SA)$ using the GJ algorithm from \Cref{lmm:proj}.
    \item Let $\{P_i:i=1\ldots{d\choose k}\}$ be the set of matrices in $\R^{d\times k}$ whose columns form all of the possible $k$-subsets of the standard basis in $\R^d$. 
    \item For every $P_i$, compute $Z_i=(BB^T)^qBP_i$, where $q=O(\epsilon^{-1}\log(d/\epsilon))$.
    \item Choose $Z$ as the $Z_i$ that minimizes $\fsnorm{B-Z_iZ_i^\dagger B}$, using \Cref{lmm:proj} to compute $Z_iZ_i^\dagger$.
    \item Compute and return the proxy loss,
    $\hat L_{k,\epsilon}(S,A)=\fsnorm{A-ZZ^\dagger B}$.
\end{enumerate}
Given $r\in\R$, we can return ``true'' if $\hat L_{k,\epsilon}(S,A)>r$ and ``false'' otherwise, obtaining a GJ algorithm that fits the assumptions of \Cref{thm:GJ}.

\begin{mylemma}\label{clm:complexity}
This GJ algorithm has degree $\Delta=O(mk\epsilon^{-1}\log(d/\epsilon))$ and predicate complexity $p\leq2^m\cdot2^{O(k)}\cdot(d/k)^{3k}$.
\end{mylemma}
\begin{proof}
Since $SA$ has $m$ rows, computing $(SA)^\dagger(SA)$ with \Cref{lmm:proj} in step 1 has degree $O(m)$ and predicate complexity $2^m$. For each $P_i$, the $q$ powering iterations in step 3 blow up the degree by $q$. Since $Z_i$ has $k$ columns, computing $Z_iZ_i^\dagger$ with (the transposed version of) \Cref{lmm:proj} blows up the degree by $O(k)$ and the predicate complexity by $2^k$. Choosing $Z$ in step 4 entails pairwise comparisons between ${d\choose k}$ values, blowing up the predicate complexity by ${{d\choose k}\choose2}\leq e^2(ed/k)^{2k}$. Step 5 blows up the degree by only $O(1)$ and does not change the predicate complexity (note that $ZZ^\dagger$ has already been computed). 
The final check whether $\hat L_{k,\epsilon}(S,A)>r$ is one of ${d\choose k}\leq (ed/k)^k$ polynomials (one per possible value of $Z$). 
Overall, the algorithm has degree $O(mkq)=O(mk\epsilon^{-1}\log(d/\epsilon))$, and predicate complexity at most $2^m\cdot2^{O(k)}\cdot(d/k)^{3k}$.
\end{proof}

Next we show that the proxy loss approximates the true SCW loss.
\begin{mylemma}\label{clm:correctness}
For every $S,A$, it holds that $0\leq \hat L_{k,\epsilon}(S,A) - L^\mathrm{SCW}_k(S,A) \leq \epsilon$.
\end{mylemma}
\begin{proof}
Given $S,A$, let $B=A(SA)^\dagger(SA)$, and let $U\in\R^{n\times k}$ a matrix whose columns are the top $k$ left-singular vectors of $B$. This means $[B]_k=UU^TB$. Therefore, by \Cref{lmm:scwalt}, we have $L^\mathrm{SCW}_k(S,A)=\fsnorm{A-UU^TB}$. 
Let $Z$ be the matrix computed in step 4 of the GJ algorithm for $\hat L_{k,\epsilon}(S,A)$, and recall we have $\hat L_{k,\epsilon}(S,A)=\fsnorm{A-ZZ^\dagger B}$.
On one hand, since $Z$ has $k$ columns, $ZZ^\dagger B$ has rank at most $k$. Therefore, the optimality of $[B]_k=UU^TB$ as a rank-$k$ approximation of $B$ implies,
\begin{equation}\label{eq:zupper}
     \fsnorm{B-UU^TB} \leq \fsnorm{B-ZZ^\dagger B}.
\end{equation}
On the other hand, by \Cref{lmm:initialp}, some $P_i$ considered in step 2 of the GJ algorithm satisfies \Cref{eq:muscocond}. By \Cref{thm:muscomusco}, this implies that its corresponding $Z_i$ (computed in step 3) satisfies
$\fsnorm{B-Z_iZ_i^\dagger B}\leq(1+\epsilon)\cdot\fsnorm{B-UU^TB}$,
and consequently, $Z$ satisfies 
\begin{equation}\label{eq:zlower}
    \fsnorm{B-ZZ^\dagger B}\leq(1+\epsilon)\cdot\fsnorm{B-UU^TB} .
\end{equation}
Since $(SA)^\dagger(SA)$ is a projection matrix, we have by the Pythagorean identity, 
\begin{align*}
    \hat L_{k,\epsilon}(S,A) &= \fsnorm{A-ZZ^\dagger B} \\
    &= \fsnorm{A-ZZ^\dagger A(SA)^\dagger(SA)} \\
    &= \fsnorm{A(SA)^\dagger(SA)-ZZ^\dagger A(SA)^\dagger(SA)} + \fsnorm{A(I-(SA)^\dagger(SA))} \\
    &= \fsnorm{B-ZZ^\dagger B} + \fsnorm{A-B} ,
\end{align*}
and similarly, $L^\mathrm{SCW}_k(S,A) = \fsnorm{A-UU^T B} = \fsnorm{B-UU^T B} + \fsnorm{A-B}$. 
Putting these together, $\hat L_{k,\epsilon}(S,A) - L^\mathrm{SCW}_k(S,A) = \fsnorm{B-ZZ^\dagger B} - \fsnorm{B-UU^T B}$. 
From \Cref{eq:zupper,eq:zlower} we now get, $0\leq \hat L(S,A) - L^\mathrm{SCW}_k(S,A) \leq \epsilon\cdot\fsnorm{B-UU^T B}$. 
The lemma follows since both $(I-UU^T)$ and $(SA)^\dagger(SA)$ are projection matrices, implying that
\[
  \fsnorm{(I-UU^T)B}\leq\fsnorm{B}=\fsnorm{A(SA)^\dagger(SA)}\leq\fsnorm{A},
\]
and we recall that we assume throughout that $\fsnorm{A}=1$.
\end{proof}

We can now complete the proof of the upper bound in \Cref{thm:main}.
Let us recall notation: 
Let $\mathcal A$ be set possible input matrices to the LRA problem (i.e., all matrices $A\in\R^{n\times d}$ with $\fsnorm{A}=1$), 
and let $\mathcal S$ be the set of all possible sketching matrices that IVY can learn (which are all matrices of order $m\times n$ with the fixed sparsity pattern used by SCW and IVY). The class of IVY losses (whose fat shattering dimension we aim to bound) is $\mathcal L_{\mathrm{IVY}}=\{L^\mathrm{SCW}_k(S,\cdot):\mathcal A\rightarrow[0,1]\}_{S\in\mathcal S}$, and the class of proxy losses is $\hat{\mathcal L}_{\epsilon}=\{\hat L_{k,\epsilon}(S,\cdot):\mathcal A\rightarrow[0,1]\}_{S\in\mathcal S}$. 

By \Cref{clm:complexity}, given $A\in\mathcal A$ and $S\in\mathcal S$, the proxy loss $\hat L_{k,\epsilon}(S,A)$ can be computed by a GJ algorithm with degree $\Delta=O(mk\epsilon^{-1}\log(d/\epsilon))$ and predicate complexity $p\leq2^m\cdot2^{O(k)}\cdot(d/k)^{3k}$. Since each $S\in\mathcal S$ is defined by $ns$ real parameters, \Cref{thm:GJ} yields that
\begin{equation}\label{eq:pdimproxy}
  \mathrm{pdim}(\hat{\mathcal L}_{\epsilon}) = O(ns\log(\Delta p)) = O(ns\cdot (m + k\log (d/k) + \log(1/\epsilon))) .
\end{equation}

Let $A_1,\ldots,A_N\in\mathcal A$ be a subset of matrices of size $N$ which is $\epsilon$-fat shattered by $\mathcal L_{\mathrm{IVY}}$. Recall that by \Cref{def:dim}, this means that there are thresholds $r_1,\ldots,r_N\in\R$, such that for every $I\subset\{1,\ldots,N\}$ there exists $S\in\mathcal S$ such that $L^\mathrm{SCW}_k(S,A_i)>r_i+\epsilon$ if $i\in I$ and $L^\mathrm{SCW}_k(S,A_i)<r_i-\epsilon$ if $i\notin I$. By \Cref{clm:correctness}, this implies that $\hat L_{k,\epsilon}(S,A_i)>r_i$ if and only if $i\in I$. Thus, $A_1,\ldots,A_N$ is pseudo-shattered by $\hat{\mathcal L}_{\epsilon}$, implying that $N\leq\mathrm{pdim}(\hat{\mathcal L}_{\epsilon})$. Since this holds for any subset of size $N$ which is $\epsilon$-fat shattered by $\mathcal L_{\mathrm{IVY}}$, we have $\mathrm{fatdim}_\epsilon(\mathcal L_{\mathrm{IVY}})\leq\mathrm{pdim}(\hat{\mathcal L}_{\epsilon})$. The upper bound in \Cref{thm:main} now follows from \Cref{eq:pdimproxy}.

\section{Other Learning-Based Algorithms in Numerical Linear Algebra}\label{sec:other}
Generally, our approach is applicable whenever the loss incurred by a given algorithm (or equivalently, by a given setting of parameters) on a fixed input can be computed by an efficient GJ algorithm (in the efficiency measures of \Cref{def:gjcomplexity}), when given ``oracle access'' to the optimal solution for that fixed input.\footnote{See Remark \ref{rem:gjoracle} regarding oracle access to the optimal solution. This oracle access was not needed for proving our bounds for IVY, but t will be needed for two of the algorithms discussed in this section: Few-shot LRA and Learned Multigrid.}
In this section, we show that in addition to IVY, our method also yields generalization bounds for various other data-driven and learning-based algorithms for problems in numerical linear algebra (specifically, LRA and regression) that have appeared in the literature. The first two algorithms discussed below (Butterfly LRA and Multi-sketch LRA) are variants of IVY, and the bounds for them essentially follow from \Cref{thm:main}. For the other two algorithms (Few-shot LRA and Learned Multigrid) we describe the requisite GJ algorithms below.


\paragraph{Butterfly learned LRA.} \cite{ailon2020sparse} suggested a learned LRA algorithm similar to IVY, except that the learned sparse sketching matrix is replaced by a dense but implicitly-sparse \emph{butterfly} gadget matrix (a known and useful gadget that arises in fast Fourier transforms). This gadget induces a sketching matrix $S\in\R^{m\times n}$ specified by $O(m\log n)$ learned parameters, where each entry is a product of $\log n$ of those parameters. Since they use the IVY loss to learn the matrix, our results give the same upper bound as \Cref{thm:main} on the fat shattering dimension of their algorithm, times $\log n$ to account for the initial degree (in the GJ sense of \Cref{def:GJ}) of the entries in $S$.
    
\paragraph{Multi-sketch learned LRA.} \cite{liu2020learned} propose a more involved learned LRA algorithm, which uses two learned sketching matrices. Still, they train each of them using the IVY loss, and therefore the upper bound from \Cref{thm:main} holds for their algorithm as well, up to constants.

\paragraph{Few-shot learned LRA.} \cite{indyk2021few} use a different loss than IVY for learned LRA: $\mathrm{loss}(S,A) = \fsnorm{U_k^TS^TSU-I_0}$, where $U$ is the left-factor in the SVD of $A$, $U_k$ is its restriction to the top-$k$ columns, and $I_0$ is the identity of order $k$ concatenated on the right with zero columns to match the number of columns in $U$. Recall that in the GJ framework, the GJ algorithm has free ``oracle access'' to $U$ (see Remark \ref{rem:gjoracle}). Therefore, this loss can be computed by GJ algorithm of degree $4$ and predicate complexity $1$, yielding an upper bounded of $O(ns)$ on its pseudo-dimension by \Cref{thm:GJ}.

%
%

\paragraph{Learned multigrid regression.} \cite{luz2020learning} presented a learning-based algorithm for linear regression, which is based on the well-studied \emph{multigrid} paradigm for solving numerical problems. 
The specific algorithm they build on is known as 2-level algebraic multigrid (AMG). 
It approximates the solution to a regression problem $\min_x\norm{Ax-b}_2^2$, where $A$ is a square matrix of order $n\times n$,
by iterative improvements that use a sparse auxiliary matrix $P\in\R^{m\times n}$ called a \emph{prolongation matrix}. 
While there is a large body of literature on heuristic methods for choosing $P$, \cite{luz2020learning} suggest to instead learn it using a graph neural network. Similarly to IVY, they keep the sparsity pattern of $P$ fixed, and learn the values of its non-zero entries as trainable parameters.

Let $x^{(i)}$ denote the approximate solution produced by 2-level AMG in iteration $i$, starting from some fixed initial guess $x^{(0)}$. 
For details of how $x^{(i)}$ is computed in practice, we refer to \cite{luz2020learning}. 
For our purposes, it suffices to note that there is a closed-form formula for obtaining $x^{(i)}$ from the true optimal solution $x^*$, described next. Let $s_1,s_2\geq 1$ be two fixed (non-learned) integer parameters of the algorithm. Let $L$ be the lower-triangular part of the input matrix $A$ (including the diagonal). \cite{luz2020learning} restrict their algorithm for input matrices $A$ and prolongation matrices $P$ such that both $L$ and $P^TAP$ are invertible. 
We then have the identity,
\begin{equation}\label{eq:amg}
    x^{(i)} = x^* + (I-L^{-1}A)^{s_2}\cdot(I-P(P^TAP)^{-1}P^TA)\cdot(I-L^{-1}A)^{s_1}\cdot(x^{(i-1)}-x^*).
\end{equation}
Let $\norm{P}_0$ denote the number of non-zeros in the fixed sparsity pattern of $P$, and recall that $m$ is the row-dimension of $P$.
We show that the pseudo-dimension of 2-level AMG with $q$ iterations is $O(\norm{P}_0\cdot q\log m)$, by describing a GJ algorithm $\Gamma$ for computing its loss $\norm{Ax^{(q)}-b}_2^2$. 
To this end, note that $\Gamma$ has ``oracle access'' to any information about $A$ (see Remark \ref{rem:gjoracle}). In particular, $(I-L^{-1}A)^{s_1}$, $(I-L^{-1}A)^{s_2}$ and $x^*$ are all available to it. 
Furthermore, as shown in the proof of \Cref{lmm:proj}, $\Gamma$ can compute $(P^TAP)^{-1}$ using degree $O(m)$ and predicate complexity zero. Using \Cref{eq:amg}, it can compute $x^{(i)}$ from $x^{(i-1)}$ while increasing the degree by only $2$ (from multiplying $(P^TAP)^{-1}$ by $P$ on the left and by $P^T$ on the right). 
Iterating this, $\Gamma$ can compute $x^{(q)}$ and $\norm{Ax^{(q)}-b}_2^2$ from the initial guess $x^{(0)}$ using degree $O(m^q)$ and  predicate complexity zero. Comparing the loss $\norm{Ax^{(q)}-b}_2^2$ to a given threshold $r$ increases the predicate complexity to $1$. 
The upper bound $O(\norm{P}_0\cdot q\log m)$ now follows from \Cref{thm:GJ}.

\subsection*{Acknowledgments}
We thank Sebastien Bubeck for helpful discussions on statistical learning, and the anonymous reviewers for useful comments. This work was supported in part by NSF TRIPODS program (award DMS-2022448); Simons Investigator Award; GIST-MIT Research Collaboration grant; MIT-IBM Watson collaboration.

\bibliographystyle{amsalpha}
\bibliography{lrapdim}

\appendix

\section{Proof of the Goldberg-Jerrum Framework (\Cref{thm:GJ})}\label{app:gj}
In this section we prove \Cref{thm:GJ}, the main theorem of the GJ framework presented in \Cref{sec:GJ}. 
The proof is taken from \cite{goldberg1995bounding} with slight modifications, and we include it here for completeness. 

We start with an intermediate result about polynomial formulas.

\begin{definition}
A \emph{polynomial formula} $f:\R^n\rightarrow\{\text{True,False}\}$ is a DNF formula over boolean predicates of the form $P(x_1,\ldots,x_n)\geq0$, where $P$ is a polynomial in the real-valued inputs to $f$.
\end{definition}

\begin{theorem}\label{thm:GJform}
Using the notation of \Cref{sec:guptaroughgarden}, suppose that each algorithm $L\in\mathcal L$ is specified by $n$ real parameters.  Suppose that for every $x\in \mathcal X$ and $r\in\R$ there is a polynomial formula $f_{x,r}$ over $n$ variables, that given $L\in \mathcal L$ checks whether $L(x)>r$. Suppose further that $f_{x,r}$ has $p$ distinct polynomials in its predicates, and each them has degree at most $\Delta$. Then, the pseudo-dimension of $\mathcal L$ is $O(n\log(\Delta p))$.
\end{theorem}

\begin{proof}
Let $x_1,\ldots,x_N$ be a shattered set of instances, and $r_1,\ldots,r_N$ the corresponding loss thresholds. We need to show $N=O(n\log(\Delta p))$.

For every $i=1,\ldots,N$, let us denote $f_i=f_{x_i,r_i}$ for brevity. Then, for a given $L\in\mathcal L$, $f_i$ checks whether $L(x_i)>r_i$. Thus, $x_1,\ldots,x_N$ being a shattered set means that the family of $N$-dimensional boolean strings $\{f_1(L),\ldots,f_N(L):L\in\mathcal L\}$ has size $2^N$.

The truth value of each $f_i$ is determined by the signs of $p$ different polynomials in $n$ variables, each of degree at most $\Delta$. Therefore, a boolean string $f_1(L),\ldots,f_N(L)$ is determined by the signs of at most $pN$ different polynomials in $n$ variables, each of degree at most $\Delta$. We now use the following classical theorem,
\begin{theorem}[\cite{warren1968lower}]\label{thm:warren}
Suppose $N\geq n$. Then $N$ polynomials in $n$ variables, each of degree at most $\Delta$, take at most $O(N\Delta/n)^n$ different sign patterns.
\end{theorem}
In our case, if $N\leq n$ then the conclusion $N=O(n\log(\Delta p))$ is trivial, thus we may assume $N>n$ and hence $pN\geq n$. Now by Warren's theorem, $2^N = \{f_1(L),\ldots,f_N(L):L\in\mathcal L\} = O(pN\Delta/n)^n$, and by solving for $N$ we get $N=O(n\log(\Delta p))$ as desired.
\end{proof}

Now we can prove \Cref{thm:GJ}.
We can describe $\Gamma_{x,r}$ by a computation tree, where arithmetic operations correspond to nodes with one child, conditional statements correspond to nodes with two children, and leaves correspond to output values. We can transform the tree into a polynomial DNF formula (as defined in the previous section) that checks whether $L(x)>r$, by ORing the ANDs of the conditional statements along each root-to-leaf computation path that ends with a ``true'' leaf.

Each branching node in the tree (i.e., a node with two children) corresponds to a conditional statement in $\Gamma_{x,r}$, which has the form of a rational predicate, meaning it determines whether $R\geq0$ for some rational function $R$ of the inputs of $\Gamma_{x,r}$. It can be easily checked that we can replace each such rational predicate with $O(1)$ polynomials predicates (that determine whether $P\geq0$ for some polynomial in the inputs of $\Gamma_{x,r}$) of degree no larger than that of $R$. Since $\Gamma_{x,r}$ has predicate complexity $p$, the obtained formula has at most $O(p)$ distinct polynomials in its predicates. Since $\Gamma_{x,r}$ has degree $\Delta$, each of these polynomials has degree at most $\Delta$. \Cref{thm:GJ} now follows from \Cref{thm:GJform}.

\section{Omitted Proofs from \Cref{sec:mk1,sec:ub}}\label{app:proofs}

\subsection{Proof of \Cref{fct:scw1dim}}\label{app:mk1}

If $w^TA=0$ then SCW returns a zero matrix, whose loss is $\fsnorm{A}$, and the statement holds. Now assume $w^TA\neq0$. We go over the steps of SCW (\Cref{alg:scw}):
\begin{itemize}
    \item Compute the row vector $w^TA$.
    \item Compute the SVD of $w^TA$. Note that for any row vector $z^T$, its SVD $U\Sigma V^T$ is given by $U$ being a $1\times1$ matrix whose only entry is $1$, $\Sigma$ being a $1\times1$ matrix whose only entry is $\norm{z}$, and $V^T$ being the row vector $\frac{1}{\norm{z}}z^T$. So for $z^T=w^TA$ we have $V^T=\frac{1}{\norm{A^Tw}}w^TA$.
    \item Compute $[AV]_1$, the best rank-$1$ approximation of $AV$. But in our case $AV$ equals $\frac{1}{\norm{A^Tw}}AA^Tw$, which is already rank $1$, so its best rank-$1$ approximation is itself, $[AV]_1=\frac{1}{\norm{A^Tw}}AA^Tw$.
    \item Return $[AV]_1V^T$, which in our case equals $\frac{1}{\norm{A^Tw}^2}AA^Tww^TA$.
\end{itemize}
So, the output rank-$1$ matrix of SCW is $\frac{1}{\norm{A^Tw}^2}AA^Tww^TA$, and its loss is as stated.

\subsection{Proof of \Cref{lmm:scwalt}}
Write the SVD of $A(SA)^\dagger (SA)$ as $U\Sigma V^T$. Note that $(SA)^\dagger (SA)$ is the orthogonal projection on the row-space of $SA$, in which every vector is a linear combination of the rows of $A$. Therefore, projecting the rows of $A$ onto the row-space of $SA$ spans all of the row-space of $SA$, or in other words, the row-spans of $A(SA)^\dagger (SA)$ and of $SA$ are the same. Consequently, the rows of $V^T$ form an orthonormal basis for the row-space of $SA$. This means in particular that $VV^T$ is the orthogonal projection on the row-space of $SA$, thus $(SA)^\dagger (SA)=VV^T$, so we can write the matrix from the lemma statement as $[A(SA)^\dagger (SA)]_k=[AVV^T]_k$. 

We recall that SCW returns the best rank-$k$ approximation of $A$ in the row-space of $SA$ (see \cite{woodruff2014sketching}), meaning it returns $ZV^T$ where $Z$ is a rank-$k$ matrix that minimizes $\fsnorm{A-ZV^T}$. So, we need to show that
\begin{equation}\label{eq:scwalt_goal}
    \forall \; Z \;\; \text{of rank $k$,} \;\;\;\; \fsnorm{A-[AVV^T]_k} \leq \fsnorm{A-ZV^T}.
\end{equation}
We use the following observation.
\begin{claim}\label{clm:scwalt_aux2}
$[AVV^T]_kVV^T=[AVV^T]_k$.
\end{claim}
\begin{proof}
Recall that $A(SA)^\dagger (SA)=AVV^T$, thus the SVD of $AVV^T$ is $U\Sigma V^T$, thus $[AVV^T]_k=U_k\Sigma_k V_k^T$. Since $V_k^TVV^T=V_k^T$ (projecting a subset of $k$ rows of $V^T$ onto the row-space of $V^T$ does not change anything), we have $[AVV^T]_kVV^T = U_k\Sigma_k V_k^TVV^T = U_k\Sigma_k V_k^T =[AVV^T]_k$.
\end{proof}

Proceeding with the proof of \Cref{lmm:scwalt}, we now have for every $Z$ of rank $k$,
\begin{align*}
    \fsnorm{A-[AVV^T]_k} &= \fsnorm{AVV^T-[AVV^T]_kVV^T} + \fsnorm{A(I-VV^T)-[AVV^T]_k(I-VV^T)} \\
    &= \fsnorm{AVV^T-[AVV^T]_k} + \fsnorm{A-AVV^T} \\
    &\leq \fsnorm{AVV^T-ZV^T} + \fsnorm{A-AVV^T} \\
    &= \fsnorm{A-ZV^T},
\end{align*}
where the first and last equalities are the Pythagorean identity (orthogonally projecting onto and against $VV^T$), the second equality is by \Cref{clm:scwalt_aux2}, and the inequality is by the optimality of $[AVV^T]_k$ as a rank-$k$ approximation of $AVV^T$ (since $Z$ has rank $k$). This proves \Cref{eq:scwalt_goal}, proving the lemma.

\subsection{Proof of \Cref{thm:muscomusco}}\label{app:muscomuscopowermethod}
In this section we explain how to read the statement of \Cref{thm:muscomusco} from \cite{musco2015randomized}. All section, theorem and page numbers refer to the arXiv version of their paper.\footnote{\url{https://arxiv.org/abs/1504.05477v4}.}

The relevant result in \cite{musco2015randomized} is the Frobenius-norm analysis of their Algorithm 1 (which they call ``Simultaneous Iteration''), stated in their Theorem 11. As they state in their Section 5.1, for the purpose of low-rank approximation in the Frobenius norm, it suffices to return $\mathrm{\mathbf{Q}}$ instead of $\mathrm{\mathbf{Z}}$ (in the notation of their Algorithm 1). Thus, steps 4--6 of Algorithm 1 can be skipped. Since (again in the notation of their Algorithm 1) $\mathrm{\mathbf{Q}}$ is the result of orthonormalizing the columns of $\mathrm{\mathbf{K}}$, we clearly have $\mathrm{\mathbf{KK}}^\dagger=\mathrm{\mathbf{QQ}}^\dagger$ (this is because $\mathrm{\mathbf{MM}}^\dagger$ is the projection matrix on the column space of any matrix $\mathrm{\mathbf{M}}$, and $\mathrm{\mathbf{K}}$ and $\mathrm{\mathbf{Q}}$ have the same column spaces). Since for our purpose we only need the projection matrix $\mathrm{\mathbf{QQ^\dagger}}$ (rather than the orthonormal basis $\mathrm{\mathbf{Q}}$), we can also skip step 3, and simply use the matrix $\mathrm{\mathbf{K}}$ as the output of their Algorithm 1, while maintaining the guarantee of their Theorem 11 (with $\mathrm{\mathbf{ZZ}}^T$ replaced by $\mathrm{\mathbf{KK}}^\dagger$). 

Next, while their Algorithm 1 chooses the initial matrix $\mathrm{\mathbf{\Pi}}$ as a random gaussian matrix, they state in their Section 5.1 that their analysis (and in particular, their Theorem 11) holds for every initial matrix $\mathrm{\mathbf{\Pi}}$ that satisfies the guarantee of their Lemma 4, which is equivalent to satisfying our \Cref{eq:muscocond}.

Finally, note that they set the number of iterations in step 1 of their Algorithm 1 to $q=O(\epsilon^{-1}\log(d))$, while we set it in \Cref{thm:muscomusco} to $q=O(\epsilon^{-1}\log(d/\epsilon))$. This is because in their setting they may assume w.l.o.g.~that $\epsilon^{-1}=\mathrm{poly}(d)$ (see bottom of their page 10), while in our setting this is not the case.

Altogether, the statement of \Cref{thm:muscomusco} follows.

\subsection{Proof of \Cref{lmm:initialp}}
For completeness, we discuss two ways to prove the lemma.

\paragraph{Proof 1.}
We start by noting that a significantly stronger version of \Cref{lmm:initialp} follows from Theorem 1.3 of \cite{deshpande2006matrix}. Slightly rephrasing (and transposing) their theorem, it states that for every matrix $B\in\R^{n\times d}$, there is a distribution (that depends on $B$) over matrices $P\in\R^{d\times k}$ whose columns are standard basis vectors, such that if we let $\tilde B_k$ be the projection of the columns of $B$ onto the columns-space of $BP$, namely $\tilde B_k=(BP)(BP)^\dagger B$, then we have
\[ \E_P\fsnorm{B-\tilde B_k} \leq (k+1)\cdot\fsnorm{B-[B]_k} . \]

In particular, there exists a supported $P$ that satisfies this equation without the expectation, yielding \Cref{lmm:initialp}.

\emph{Remark.}
Note that this is in fact a quantitatively stronger version of \Cref{lmm:initialp}, since the $O(kd)$ term in \Cref{eq:muscocond} is replaced here by $(k+1)$, which \cite{deshpande2006matrix} furthermore show is the best possible. However, this improvement does not strengthen the final bounds we obtain in our theorems. This is because the analysis of the power method (\Cref{thm:muscomusco}) requires the $\log d$ term in the number of iterations $q$ even if the initial approximation (\Cref{eq:muscocond}) is up to a factor of $(k+1)$ instead of $O(kd)$. Technically, this stems from the analysis in \cite{musco2015randomized}: In the equation immediately after their eq. (4) on page 11, the $\log d$ term is needed not just to eliminate $d$ from the numerator, but also to gain a $d^{O(1)}$ term in the denominator.

\paragraph{Proof 2.}
Since the aforementioned result of \cite{deshpande2006matrix} is difficult and stronger than we require, we now also give a more basic proof of \Cref{lmm:initialp}, for completeness.

Let $B\in\R^{n\times d}$ be written in SVD form as $B=U\Sigma V^T$. 
Let $V_k^T$ be the matrix with the top $k$ rows of $V^T$, and $V^T_{-k}$ the matrix with the remaining rows. 
We use the following fact from \cite{woodruff2014sketching}, which originates in \cite{boutsidis2014near} and is also used in \cite{musco2015randomized} (see their Lemma 14).
\begin{mylemma}\label{lmm:boutsidis}
Let $P\in\R^{d\times k}$ be any matrix such that the $k\times k$ matrix $V_k^TP$ has rank $k$. Then,
\[
  \fsnorm{B-(BP)(BP)^\dagger B} \leq \fsnorm{B-[B]_k} \cdot \left(1 + \norm{V^T_{-k}P}_2^2 \cdot \norm{(V_k^TP)^\dagger}_2^2 \right).
\]
\end{mylemma}

We will construct a matrix $P\in\R^{d\times k}$ whose columns are distinct standard basis vectors of $\R^d$. Since it would thus have orthonomal columns, as does $V_{-k}$, we would have $\norm{V_{-k}^TP}_2^2\leq \norm{V_{-k}^T}_2^2\cdot \norm{P}_2^2=1$. 
Therefore, using \Cref{lmm:boutsidis}, for $P$ to satisfy \Cref{eq:muscocond} and thus prove \Cref{lmm:initialp}, it suffices to construct it such that $V_k^TP$ has rank $k$ and satisfies $\norm{(V_k^TP)^\dagger}_2^2\leq d$.
This is equivalent to showing that the smallest of the $k$ singular values of $V_k^TP$ is at least $1/\sqrt{d}$, which is what we do in the rest of the current proof.

We construct $P$ by the following process. Initialize $Z_1\in\R^{k\times d}$ as $Z_1\leftarrow V_k^T$, and a zero matrix $P\in\R^{d\times k}$. For $i=1,\ldots,k$:
\begin{enumerate}
    \item Let $z^i_1,\ldots,z^i_d$ denote the columns of $Z_i$.
    \item Let $j_i\leftarrow \mathrm{argmax}_{j\in[d]}\norm{z^i_j}_2^2$. 
    \item Set column $i$ of $P$ to be $e_{j_i}$.
    \item Let $\Pi_i$ be the orthogonal projection matrix against $\mathrm{span}(v_{j_1},\ldots,v_{j_i})$.
    \item $Z_{i+1}\leftarrow \Pi_i V_k^T$.
\end{enumerate}
Let $v_1,\ldots,v_d$ denote the columns of $V_k^T$. Observe that the columns of $V_k^TP$ are $v_{j_1},\ldots,v_{j_k}$. 

\begin{claim}\label{clm:zlargenorm}
For every $i=1,\ldots,k$ we have $\norm{z^i_{j_i}}_2^2\geq \frac{k-i+1}{d}$.
\end{claim}
\begin{proof}
Let $i\in[k]$. Since $V_k^T$ has orthonormal rows, each of its $k$ singular values equals $1$, and any best rank-$(i-1)$ approximation of it, $[V_k^T]_{i-1}$, satisfies $\fsnorm{V_k^T-[V_k^T]_{i-1}} = k-(i-1)$. Since $\Pi_{i-1}$ is a projection against $i-1$ directions, $(I-\Pi_{i-1})V_k^T$ can be viewed as a rank-$(i-1)$ approximation of $V_k^T$, and therefore
\[ \fsnorm{Z_i} = \fsnorm{\Pi_{i-1}V_k^T} = \fsnorm{V_k^T - (I-\Pi_{i-1})V_k^T} \geq \fsnorm{V_k^T-[V_k^T]_{i-1}} = k-i+1. \]
Hence the average squared-$\ell_2$ mass of columns in $Z_i$ is at least $\frac{k-i+1}{d}$, and hence the column with maximal $\ell_2$-norm --- which we recall is defined to be $z_{j_i}^i$ --- has squared $\ell_2$-norm at least $\frac{k-i+1}{d}$.
\end{proof}

\begin{claim}\label{clm:zspanned}
For every $i=1,\ldots,k$, $z^i_{j_i}$ is spanned by $v_{j_1},\ldots,v_{j_i}$.
\end{claim}
\begin{proof}
Observe that $z^i_{j_i}=\Pi_{i-1}v_{j_i}$ (with the convention that $\Pi_0$ is the identity). Therefore, $z^i_{j_i}=v_{j_i}-(I-\Pi_{i-1})v_{j_i}$, and the claim follows by noting that $I-\Pi_{i-1}$ is the orthogonal projection onto the subspace spanned by $v_{j_1},\ldots,v_{j_{i-1}}$.
\end{proof}

\begin{claim}\label{cor:zortho}
$z^1_{j_1},\ldots,z^k_{j_k}$ are pairwise orthogonal.
\end{claim}
\begin{proof}
Let $i<i'$. By \Cref{clm:zspanned} $z^i_{j_i}$ is spanned by $v_{j_1},\ldots,v_{j_i}$. At the same time, $z^{i'}_{j_{i'}}$ is a column of the matrix $Z_{i'}$, whose columns have been orthogonally projected against $v_{j_1},\ldots,v_{j_{i'-1}}$, a set that contains $v_{j_1},\ldots,v_{j_i}$. 
Thus $z^{i'}_{j_{i'}}$ is orthogonal to $z^i_{j_i}$.
\end{proof}

\begin{claim}\label{cor:zqr}
Let $i\in[k]$. Then $v_{j_i}$ can be written uniquely as a linear combination of $z_{j_1}^1,\ldots,z_{j_i}^i$, such that the coefficient of $z_{j_i}^i$ is $1$.
\end{claim}
\begin{proof}
We have shown in \Cref{cor:zortho} that $z^1_{j_1},\ldots,z^k_{j_k}$ are orthogonal and in \Cref{clm:zlargenorm} that each is non-zero, so they form a basis of $\R^k$. Thus $v_{j_i}$ is written uniquely as a linear combination of them. 
As noted in the proof of \Cref{clm:zspanned}, we have $z^i_{j_i}=\Pi_{i-1}v_{j_i}=v_{j_i}-(I-\Pi_{i-1})v_{j_i}$, or equivalently $v_{j_i}=z^i_{j_i}+(I-\Pi_{i-1})v_{j_i}$. 

We recall that $(I-\Pi_{i-1})$ is the orthogonal projection onto the subspace $W=\mathrm{span}(v_{j_1},\ldots,v_{j_{i-1}})$, whose dimension is at most $i-1$. 
By \Cref{clm:zspanned}, $W$ contains $z_{j_1},\ldots,z_{j_{i-1}}$. 
Since $z_{j_1},\ldots,z_{j_k}$ form a basis of $\R^k$, we now get that $W$ is spanned by $z_{j_1},\ldots,z_{j_{i-1}}$, and cannot contain any $z_{j_{i'}}$ with $i'\geq i$. In particular, $(I-\Pi_{i-1})v_{j_i}$ is written uniquely as a linear combination of $z_{j_1},\ldots,z_{j_{i-1}}$. Recalling that $v_{j_i}=z^i_{j_i}+(I-\Pi_{i-1})v_{j_i}$, the claim follows.
\end{proof}

We can finally complete the proof of \Cref{lmm:initialp}. 
For every $i\in[k]$, let $q_i$ denote the unit-length vector in the direction of $z^i_{j_i}$, i.e., $q_i=\frac{1}{\norm{z^i_{j_i}}}z^i_{j_i}$. 
Let $Q\in\R^{k\times k}$ be the matrix whose columns are $q_1,\ldots,q_k$. \Cref{cor:zqr} means that we can write $V_k^TP$ (since its columns, we recall, are $v_{j_1},\ldots,v_{j_k}$) as $V_k^TP=QR$, where $R$ is an upper triangular matrix, whose diagonal entries are $\norm{z^i_{j_i}}$ for $i=1,\ldots,k$. 
Since $Q$ has orthonormal columns (from \Cref{cor:zortho} and the fact that we normalized its columns), this is a QR-decomposition of $V_k^TP$. Hence, $V_k^TP$ and $R$ have the same singular values. 
The diagonal entries of $R$ are its eigenvalues, which are also its singular values as all are non-negative. By \Cref{clm:zlargenorm}, each diagonal entry of $R$ is at least $1/\sqrt{d}$. 
This implies the same lower bound on the smallest singular value of $R$ and thus of $V_k^TP$, which as explained earlier, implies \Cref{lmm:initialp}.

\section{Proof of the Lower Bound}\label{app:lb}
In this section we prove the lower bound in \Cref{thm:main}, restated next. 
\begin{theorem}
For every $s\leq k\leq m$ and $\epsilon\in(0,\tfrac1{2\sqrt k})$, the $\epsilon$-fat shattering dimension of IVY with target low rank $k$, sketching dimension $m$ and sparsity $s$ is $\Omega(ns)$.
\end{theorem}
\begin{proof}
We may assume w.l.o.g.~that $m=k$, since we can always augment a sketching matrix with $k$ rows with $m-k$ additional zero rows, without changing the result of SCW.

We start with the special case $s=k$, where the sketching matrix $S$ is allowed to be fully dense, and the desired lower bound is $\Omega(nk)$. Let $A_0$ be the $n\times k$ matrix whose $i$-th column is $e_i$ (the standard basis vector) for all $i=1\ldots k$. 
For every $i\in\{1\ldots k\}$ and $t\in \{k+1\ldots n\}$, Let $A_{(i,t)}$ be given by replacing the $i$-th column of $A_0$ with $e_t$. Observe that each $A_{(i,t)}$ has rank $k$, and each of its singular values equals $1$. We will show that the set of matrices $Z=\{A_{(i,t)}:i=1\ldots k,\; t=k+1\ldots n\}$ is shattered, which would imply the lower bound since its size is $k(n-k)=\Omega(nk)$.

Let $Z'\subset Z$. It suffices to exhibit a sketching matrix $S\in\R^{k\times n}$ (with unbounded sparsity, since $s=k$) such that for every $A_{(i,t)}$, the SCW loss of using $S$ for a rank-$k$ approximation of $A$ is $0$ if $A_{(i,t)}\in Z'$, and at least $1$ otherwise. 
We set the first $k$ columns of $S$ to be the order-$k$ identity matrix. Then, for every $t=k+1,\ldots ,n$, let $J_t = \{i\in {1\ldots k} : A_{(i,t)}\in Z'\}$. In the $t$-th column of $S$, we put $1$'s in rows $J_t$ and $0$'s in the remaining rows. 

Fix $A_{(i,t)}$. Recall that SCW returns the best rank-$k$ approximation of $A$ in the row-space of $SA_{(i,t)}$, which is a $k \times k$ matrix.

\begin{itemize}
    \item Suppose $A_{(i,t)}\notin Z'$. 
In this case, we argue that the $i$-th row of $SA_{(i,t)}$ is zero. This implies that the row-space of $SA_{(i,t)}$ has dimension at most $k-1$, so SCW incurs loss at least 1.

Indeed, row $i$ of $SA_{(i,t)}$ equals the sum of rows $j$ of $A_{(i,t)}$ for which $S(i,j)=1$. The only non-zero rows in $A_{(i,t)}$, by construction, are rows $\{1...k\}\setminus\{i\}$ and row $t$. Since the first $k$ columns of $S$ are the identity, $S(i,j)=0$ for every $j\in\{1...k\}\setminus\{i\}$. Since $i\notin  J_t$, then by construction of $S$ we have $S(i,t)=0$. Overall, the $i$-th row of $SA_{(i,t)}$ is zero.

\item Suppose $A_{(i,t)}\in Z'$. In this case, we argue that the row-space of $SA_{(i,t)}$ has dimension $k$. Since the ambient row-dimension is also $k$, this means SCW returns a perfect rank-$k$ approximation, i.e., zero loss.

Indeed, again, the only non-zero rows of $A_{(i,t)}$ are rows $\{1...k\}\setminus\{i\}$ and row $t$. But in this case $i\in J_t$, thus $S(i,t)=1$ and thus row $i$ of $SA_{(i,t)}$ equals $e_i$. For every $j\in\{1...k\}\setminus\{i\}$, row $j$ of $SA_{(i,t)}$ equals either $e_j+e_i$ (if $j\in J_t$) or just $e_j$ (if $j\notin J_t$). It is thus clear that the rows of $SA_{(i,t)}$ span all of $\R^k$. 
\end{itemize}
In conclusion, the SCW loss is $0$ on matrices in $Z'$ at least $1$ on matrices in $Z\setminus Z'$. 

Next we extend it to an $\Omega(ns)$ lower bound for any $s=1,\ldots,k$. We assume for simplicity that $s$ is an integer divisor of $k$ and that $k$ is an integer divisor of $ns$.

We partition an input matrix $A\in\R^{n\times k}$ into $k/s$ diagonal blocks of order $(ns/k)\times s$ each (everything outside the diagonal blocks is zero). To construct the shattered set, we first set each diagonal block in $A$ to $A_0\in\R^{(ns/k)\times s}$ (as defined above), then choose one ``critical'' block $b\in\{1,\ldots,k\}$ and replace its $A_0$ with $A_{(i,t)}$ as defined above, with $i\in\{1,\ldots,s\}$ and $t\in\{s+1,\ldots,ns/k\}$. Denote the resulting matrix by $A_{(b,i,t)}$. The total size of the shattered set $\{A_{(b,i,t)}\}$ is $\tfrac{k}{s}\cdot s\cdot(\tfrac{ns}{k}-s)=(n-k)s=\Omega(ns)$. The corresponding sketching matrix, arising from the dense construction above, has block-diagonal structure with blocks of order $s\times (ns/k)$ each, and in particular has at most $s$ nonzeros per column, as needed.

The correctness of the construction follow immediately from the dense case. In more detail, let $Z''$ be a subset of the shattered set $\{A_{(b,i,t)}\}$. Fix $A_{(b,i,t)}$ (not necessarily in $Z''$). The $k\times k$ matrix $SA_{(b,i,t)}$ has block-diagonal structure with blocks of order $s\times s$. Each non-critical block equals the order-$s$ identity, while the critical block, by the proof of the dense case above, equals the order-$s$ identity if $A_{(b,i,t)}\in Z''$ and has rank at most $s-1$ otherwise. Therefore, if $A_{(b,i,t)}\in Z''$ then SCW with sketching matrix $S$ finds a zero-loss rank-$k$ approximation of every $A_{(b,i,t)}\in Z''$, and incurs loss at least $1$ for every $A_{(b,i,t)}\notin Z''$. 

Finally, recall that we need to normalize our matrices to have squared Frobenius norm $1$. In the above construction it is instead $k$, so we need to divide each entry by $1/\sqrt k$. This means that the loss is zero for matrices in the shattered set, and is at least $1/\sqrt k$ for matrices outside the shattered set. Thus, the set is $\epsilon$-fat shattered for every $\epsilon\in(0,\tfrac1{2\sqrt k})$.
\end{proof}

\section{Proof of \Cref{thm:samples}}\label{app:samples}
\subsection{Uniform Convergence and ERM Upper Bound}
By classical results in learning theory, the fat shattering dimension can be used to obtain an upper bound on the number of samples needed for uniform convergence and for ERM learning (see for example Theorem 19.1 in \cite{anthony2009neural}), and we could simply plug \Cref{thm:main} into these results. 
However, we can get a sharper bound by exploiting the proxy loss from \Cref{sec:proxy}, since for the latter we have an upper bound on the pseudo-dimension, which yields sharper bounds for ERM learning than the fat shattering dimension. 


Let $\ell_{\mathrm{IVY}}=\ell_{\mathrm{IVY}}(\epsilon,\delta)$ be the number of samples required for $(\epsilon,\delta)$-uniform convergence for the family IVY losses, $\mathcal L_{\mathrm{IVY}}=\{L_k^{\mathrm{SCW}}(S,\cdot):\mathcal A\rightarrow[0,1]\}_{S\in\mathcal S}$. 
Let $\hat\ell=\hat\ell(\epsilon,\delta)$ be the number of samples required for $(\epsilon,\delta)$-uniform convergence for the family of proxy losses, $\hat{\mathcal L}_{\epsilon}=\{\hat L_{k,\epsilon}(S,\cdot):\mathcal A\rightarrow[0,1]\}_{S\in\mathcal S}$. 
Fix $\hat\ell$ matrices $A_1,\ldots,A_{\hat\ell}\in\mathcal A$. \Cref{clm:correctness} implies
\[ \forall \; S\in\mathcal S, \;\;\;\; \left| \frac{1}{\hat\ell}\sum_{i=1}^{\hat\ell}\hat L_{k,\epsilon}(S,A_i) - \frac{1}{\hat\ell}\sum_{i=1}^{\hat\ell}L_k^{\mathrm{SCW}}(S,A_i) \right| \leq \epsilon \]
and 
\[ \forall \; S\in\mathcal S, \;\;\;\; \left| \E_{A\sim\mathcal D}[\hat L_{k,\epsilon}(S,A)] - \E_{A\sim\mathcal D}[L_k^{\mathrm{SCW}}(S,A)] \right| \leq \epsilon . \]
Let $\mathcal D$ be a distribution over $\mathcal A$. 
By definition of $\hat\ell$, with probability at least $1-\delta$ over sampling $A_1,\ldots,A_{\hat\ell}$ independently from $\mathcal D$, we have 
\[
\forall \; S\in\mathcal S, \;\;\;\; \left| \frac{1}{\hat\ell}\sum_{i=1}^{\hat\ell}\hat L_{k,\epsilon}(S,A_i) - \E_{A\sim\mathcal D}[\hat L_{k,\epsilon}(S,A)]\right| \leq \epsilon ,
\]
which combined with the above, implies
\[ \forall \; S\in\mathcal S, \;\;\;\; \left| \frac{1}{\hat\ell}\sum_{i=1}^{\hat\ell}L_k^{\mathrm{SCW}}(S,A_i) - \E_{A\sim\mathcal D}[L_k^{\mathrm{SCW}}(S,A)]\right| \leq 3\epsilon  . \]
Consequently, $\ell_{\mathrm{IVY}}(3\epsilon,\delta)\leq\hat\ell(\epsilon,\delta)$. 
By classical results in learning theory (see, e.g., Theorem 3.2 in \cite{gupta2017pac}), the sample complexity can be upper bounded using the pseudo-dimension, and in particular, $\hat\ell(\epsilon,\delta) =  O(\epsilon^{-2}\cdot(\mathrm{pdim}(\hat{\mathcal L}_{\epsilon})+\log(1/\delta)))$. 
By \Cref{eq:pdimproxy} in \Cref{sec:proxy}, $\mathrm{pdim}(\hat{\mathcal L}_{\epsilon})=O(ns\cdot (m + k\log (d/k) + \log(1/\epsilon)))$. 
Putting everything together, 
\[ \ell_{\mathrm{IVY}}(3\epsilon,\delta) = O(\epsilon^{-2}\cdot (ns\cdot (m + k\log (d/k) + \log(1/\epsilon)) + \log(1/\delta))) . \]
As a consequence, that many samples suffice for $(6\epsilon,\delta)$-learning IVY with ERM. The upper bound in \Cref{thm:samples} follows by scaling $\epsilon$ down by a constant.

\subsection{Uniform Convergence Lower Bound}
We proceed to the lower bound on the number of samples needed for IVY to admit $(\epsilon,\epsilon)$-uniform convergence. 
It relies on some known results about the connection between uniform convergence and the fat shattering dimension, which we now detail. 

We introduce some notation. 
Let $\mathcal L$ be a class of functions $L:\mathcal X\rightarrow[0,1]$, and let $\mathcal D$ be a distribution over $\mathcal X$. For every $L\in\mathcal L$, denote its expected loss over $\mathcal D$ by
\[ z(L):=\E_{x\sim\mathcal D}[L(x)] . \]
Given $\ell$ i.i.d.~samples $(x_1,\ldots,x_\ell)\sim\mathcal D^\ell$, denote the empirical loss over the samples by
\[ \hat z_\ell(L):=\frac1\ell\sum_{i=1}^\ell L(x_i) . \]
Suppose $\mathcal L$ admits $(\epsilon,\epsilon)$-uniform convergence with $\ell$ samples. This can now be written as
\begin{equation}\label{eq:epsepsuni}
  \Pr_{\mathcal D^\ell}\left[\sup_{L\in\mathcal L}\left|\hat{z}_\ell(L) - z(L)\right|\leq\epsilon\right]\geq1-\epsilon .
\end{equation}
Our goal for this section is to prove a lower bound on $\ell$. We begin the proof. 
\Cref{eq:epsepsuni} implies 
\begin{equation}\label{eq:expuni}
  \E_{\mathcal D^\ell}\left[\sup_{L\in\mathcal L}\left|\hat{z}_\ell(L) - z(L)\right|\right]\leq 2\epsilon . 
\end{equation}
This expectation can be bounded using Rademacher averages. Define the centered class $\mathcal L_c$ of $\mathcal L$ as
\[ \mathcal L_c := \{L-z(L) : L\in\mathcal L\} . \]
For the given sample $(x_1,\ldots,x_\ell)$, define the Rademacher average of $\mathcal L_c$ as
\[
  \mathrm{Rad}_\ell(\mathcal L_c) := \E_{\sigma_1,\ldots,\sigma_\ell}\left[\sup_{L'\in\mathcal L_c}\left|\frac1\ell\sum_{i=1}^\ell\sigma_i L'(x_i)\right|\right] = \E_{\sigma_1,\ldots,\sigma_\ell}\left[\sup_{L\in\mathcal L}\left|\frac1\ell\sum_{i=1}^\ell\sigma_i\left(L(x_i) - z(L)\right)\right|\right] ,
\]
where each of $\sigma_1,\ldots,\sigma_\ell$ is independently uniform in $\{1,-1\}$. 
The desymmetrization inequality for Rademacher processes (see \cite{koltchinskii2006local}) states that
\begin{equation}\label{eq:desymm}
  \E_{\mathcal D^\ell}\left[\sup_{L\in\mathcal L}\left|\hat{z}_\ell(L) - z(L)\right|\right]
  \geq \frac12\E_{\mathcal D^\ell} \left[\mathrm{Rad}_\ell(\mathcal L_c)\right] .
\end{equation}
A version of Sudakov's minorization inequality for Rademacher processes (Equation (26) in \cite{steinwart2015measuring}, based on Corollary 4.14 in \cite{ledoux2013probability}) states that
\begin{equation}\label{eq:sudakov}
  \mathrm{Rad}_\ell(\mathcal L_c)
  \geq  \frac{\alpha_1}{\sqrt\ell} \cdot \left(\ln\left(2+\frac{\alpha_2}{\Psi_\ell(\mathcal L_c)}\right)\right)^{-1/2} \cdot \sup_{\gamma>0}\gamma \sqrt{\ln\mathcal N_2(\gamma,\mathcal L,\ell)} ,
\end{equation}
where $\alpha_1,\alpha_2>0$ are absolute constants, $\mathcal N_2(\epsilon,\mathcal L,\ell)$ is the \emph{covering number}, and
\[
  \Psi_\ell(\mathcal L_c) := \sup_{L'\in\mathcal L_c}\sqrt{\frac1\ell\sum_{i=1}^\ell(L'(x_i))^2}.
\]
We forgo the definition of the covering number, since we only need the standard fact that it can be lower-bounded in terms of the fat shattering dimension. It is encompassed in the following claim.
\begin{claim}\label{clm:coveringnumber}
Let $\gamma\in[4\epsilon,1)$ and $\delta\in(0,\tfrac1{100})$. Suppose $\mathcal L$ admits $(\epsilon,\delta)$-uniform convergence with $\ell$ samples, and $\mathrm{fatdim}_{16\gamma}(\mathcal L)>1$. Then $\ln\mathcal N_2(\gamma,\mathcal L,\ell)\geq\tfrac18\cdot\mathrm{fatdim}_{16\gamma}(\mathcal L)$.
\end{claim}
\begin{proof}
Let $\ell'$ be the number of samples required for $(\tfrac12\gamma,\delta)$-learning $\mathcal L$ (see \Cref{sec:guptaroughgarden} for the definition of $(\epsilon,\delta)$-learning). Theorem 19.5 in \cite{anthony2009neural} gives the lower bound $\ell'\geq\tfrac1{16\alpha}(\mathrm{fatdim}_{\gamma/(2\alpha)}(\mathcal L)-1)$ for any $0<\alpha<1/4$. Setting $\alpha=1/32$ and using $\mathrm{fatdim}_{16\gamma}(\mathcal L)>1$ yields $\ell'\geq2\cdot\mathrm{fatdim}_{16\gamma}(\mathcal L)-2 \geq \mathrm{fatdim}_{16\gamma}(\mathcal L)$. Since $(\epsilon,\delta)$-uniform convergence implies $(2\epsilon,\delta)$-learning and thus $(\tfrac12\gamma,\delta)$-learning (as $\gamma\geq4\epsilon$), we have $\ell\geq\ell'\geq \mathrm{fatdim}_{16\gamma}(\mathcal L)$. Now we can apply Lemma 10.5 and Theorem 12.10 from \cite{anthony2009neural}, which yield $\ln\mathcal N_2(\gamma,\mathcal L,\ell) \geq \ln\mathcal N_1(\gamma,\mathcal L,\ell) \geq \tfrac18\cdot\mathrm{fatdim}_{16\gamma}(\mathcal L)$.
\end{proof}
Putting together everything so far (i.e., combining \Cref{eq:expuni,eq:desymm,eq:sudakov,clm:coveringnumber}), we get
\[
  2\epsilon \geq \frac12\E_{\mathcal D^\ell} \left[\frac{\alpha_1}{8\sqrt\ell} \cdot \left(\ln\left(2+\frac{\alpha_2}{\Psi_\ell(\mathcal L_c)}\right)\right)^{-1/2} \cdot \sup_{\gamma\in[4\epsilon,1)}\gamma \sqrt{\mathrm{fatdim}_{16\gamma}(\mathcal L)}\right] ,
\]
provided that $\mathrm{fatdim}_{16\gamma}(\mathcal L)>1$ (as needed for \Cref{clm:coveringnumber}).

The derivations so far have been for any $\mathcal L$ that admits $(\epsilon,\epsilon)$-uniform convergence with $\ell$ samples. Now we begin specializing the arguments for IVY, that is, for $\mathcal L=\mathcal L_{\mathrm{IVY}}$. 
We choose $\gamma=1/(64\sqrt{k})$. Note that the assumption in \Cref{thm:samples} is that $\epsilon\leq1/(256\sqrt{k})$, ensuring that $\gamma\geq4\epsilon$. By the lower bound in \Cref{thm:main}, $\mathrm{fatdim}_{16\gamma}(\mathcal L)=\Omega(ns)$. Plugging these above, we get
\begin{equation}\label{eq:epspsi}
  \epsilon \geq \Omega(1) \cdot \E_{\mathcal D^\ell} \left[ \sqrt{\frac{ns}{\ell k} \cdot \left(\ln\left(2+\frac{\alpha_2}{\Psi_\ell(\mathcal L_c)}\right)\right)^{-1}} \right] .
\end{equation}
It remains to bound $\Psi_\ell(\mathcal L_c)$ from below. We show it is lower-bounded by a constant even for very simple input distributions for IVY, supported on only two matrices (and indeed, for any distribution such that there is a loss function $L\in\mathcal L$ with loss $0$ on half the distribution mass and loss $1$ on the other half). Recall that
\[
  \Psi_\ell(\mathcal L_c) := \sup_{L'\in\mathcal L_c}\sqrt{\frac1\ell\sum_{i=1}^\ell(L'(x_i))^2} = \sup_{L\in\mathcal L}\sqrt{\frac1\ell\sum_{i=1}^\ell(L(x_i)-z(L))^2} .
\]
Let $e_1,\ldots,e_n\in\R^n$ be the standard basis in $\R^n$. Let $A_0,A_0',A_1,A_1'\in\R^{n\times d}$ be defined as follows: the first $k$ columns of $A_0'$ are $e_1,\ldots,e_k$, and the rest are zero; the first $k$ columns of $A_1'$ are $e_{k+1},\ldots,e_{2k}$, and the rest are zero; $A_0=\frac1{\sqrt k}A_0'$; and $A_1=\frac1{\sqrt k}A_1'$. Note that $\fsnorm{A_0}=\fsnorm{A_1}=1$. Let $\mathcal D$ be the uniform distribution over $\{A_0,A_1\}$. Let $L_0\in\mathcal L$ be the IVY loss function induced by the sketching matrix $S_0\in\R^{k\times n}$ whose rows are $e_1^T,\ldots,e_k^T$. (In the notation of the previous sections, $L_0=L_k^{\mathrm{SCW}}(S_0,\cdot)\in\mathcal L_{\mathrm{IVY}}$.) It is not hard to see that $L_0(A_0)=0$ and $L_0(A_1)=1$. In particular, $z(L_0)=\tfrac12$. Therefore, for any sample $x_1,\ldots,x_\ell$ from $\mathcal D$,
\[
  \Psi_\ell(\mathcal L_c) \geq \sqrt{\frac1\ell\sum_{i=1}^\ell(L_0(x_i)-z(L_0))^2} = \sqrt{\frac1\ell\sum_{i=1}^\ell\left(\frac12\right)^2} = \frac12.
\]
The proof is now easily completed. Plugging the above bound on $\Psi_\ell(\mathcal L_c)$ into \Cref{eq:epspsi} yields
$\epsilon \geq \Omega(1) \cdot \E_{\mathcal D^\ell}\left[\sqrt{ns/(\ell k)}\right]$. Since $n,s,\ell,k$ are constants w.r.t.~the sample from $ D^\ell$, we can dispense with the expecation and write $\epsilon \geq \Omega(\sqrt{ns/(\ell k)})$. Rearranging yields $\ell \geq \Omega(\epsilon^{-2}ns/k)$.

\subsection{General Learning Lower Bound}
The lower bound $\Omega(\epsilon^{-1}+ns)$ on the number of samples needed for $(\epsilon,\delta)$-learning IVY with any learning procedure follows by plugging the lower bound on the fat shattering dimension from \Cref{thm:main} into Theorem 19.5 in \cite{anthony2009neural}.

\section{Connection to Prior Work}\label{app:related}
In this section we discuss the connection of our techniques to prior techniques for proving statistical generalization bounds for data-driven algorithms, and specifically to \cite{gupta2017pac} and \cite{balcan2021much}. The goal is to place our work in the context of related work, and also to explain why previous techniques do not give useful generalization bounds for the linear algebraic algorithms considered in this paper. 

\subsection{Illustrative Example 1: Learning Greedy Heuristics for Knapsack}\label{app:knapsack}
\cite{gupta2017pac} discuss learned heuristics for the Knapsack problem as an illustrative example for their statistical learning framework. 
Recall that in the Knapsack problem, the input is $n$ items with values $v_1,\ldots,v_n>0$, respective costs $c_1,\ldots,c_n>0$, and a cost limit $C>0$. The goal is to return $I\subset\{1,\ldots,n\}$ that maximizes the total value $\sum_{i\in I}v_i$ under the cost constraint $\sum_{i\in I}c_i<C$. This problem is well-known to be NP-hard.

\cite{gupta2017pac} consider the following family of greedy heurstics for Knapcask. Given a parameter $\rho\in R$, define the rank of item $i$ as $v_i/c_i^\rho$, and let $L_\rho$ be the greedy heuristic that adds the highest ranked items to $I$ as long as the cost limit is not exceeded. They then let $\rho$ be a learnable parameter, and prove that the pseudo-dimension of the class of heuristics $\mathcal L=\{L_\rho\}_{\rho\in\R}$ is $O(\log n)$.

The crux of the proof is the observation that for a given instance, the ``loss'' (or in this case the utility, since this is a combinatorial maximization problem)
of a solution $I$ is fully determined by the result of the ${n\choose2}$ comparison $v_i/c_i^\rho\geq^{?}v_j/c_j^\rho$, or equivalently $\rho\geq^{?}\log(v_j/v_i)/\log(c_j/c_i)$. This means that the parameter space $\R$ is partitioned into ${n\choose2}$ intervals such that the ``loss'' (utility) of the given instance is constant on each interval. The pseudo-dimension can be then bounded by the log of the number of intervals. The duality-based framework developed by \cite{balcan2021much} is a vast generalization of this argument, and recovers the same bound for Knapsack. 

The GJ framework presented in \Cref{sec:GJ} recovers it as well, for the same underlying reason. Since the ``loss'' (utility) of a given instance is determined by the ${n\choose2}$ comparisons $\rho\geq^{?}\log(v_j/v_i)/\log(c_j/c_i)$, which are polynomial predicates of degree $1$ in the variable $\rho$, it can be computed by a GJ algorithm with degree $1$ and predicate complexity ${n\choose 2}$. The upper bound $O(\log n)$ on the pseudo-dimension thus follows from \Cref{thm:GJ}.\footnote{Note that each heuristic $L_\rho$ is specified by the single real parameter $\rho$, so $n$ in the notation of \Cref{thm:GJ} is $1$.}

\subsection{Illustrative Example 2: IVY in the Case $m=k=1$}
In order to compare our techniques with the duality-based framework of \cite{balcan2021much}, it is instructive to consider the simple case $m=k=1$ of IVY. 
In \Cref{sec:mk1}, we proved a tight bound of $O(n)$ on the pseudo-dimension of IVY in this case. 

Our proof showed that for a given input matrix $A\in\R^{n\times d}$, the loss equals $\fsnorm{A}$ if the sketching vector $w\in\R^n$ satisfies $w^TA=0$, and equals $0$ otherwise. By the main definition of \cite{balcan2021much}, this means that the dual class of losses is $(\mathcal F,\mathcal G,d)$-piecewise decomposable, where $\mathcal F$ is the class of constant-valued functions, $\mathcal G$ is the class of $n$-dimensional linear threshold functions, and  $d$ is the column-dimension of $A$ (note that $d$ is the number of functions from $\mathcal G$ involved in the condition $w^TA=0$). Denoting the dual classes of $\mathcal F$ and $\mathcal G$ by $\mathcal F^*$ and $\mathcal G^*$ respectively (see \cite{balcan2021much} for the definition of dual classes), we have $\mathrm{pdim}(\mathcal F^*)=0$ and $\mathrm{VCdim}(\mathcal G^*)=n+1$. Therefore, the main theorem of \cite{balcan2021much} gives a bound of $O(n\log n)$ on the pseudo-dimension of IVY with $m=k=1$, which is looser than the tight bound by $\log n$. 

We remark that since the condition $w^TA=0$ is equivalent to $\norm{w^TA}^2=0$, the dual class is also $(\mathcal F,\mathcal G_2,1)$-piecewise decomposable where $\mathcal G_2$ is a class of quadratic threshold functions in $n$ variables. It can be checked that $\mathrm{VCdim}(\mathcal G_2^*)=\tfrac12(n+1)(n+2)$, leading to an even looser bound of $O(n^2\log n)$ on the pseudo-dimension.

\subsection{The General Case}
Finally, let us point out a formal connection between the GJ framework and the duality framework of \cite{balcan2021much}. The proof of \Cref{thm:GJ} in fact shows that if a class of algorithms admits a GJ algorithm with degree $\Delta$ and predicate complexity $p$ for computing the loss, then the dual class as defined by \cite{balcan2021much} is $(\mathcal F,\mathcal G,p)$-piecewise decomposable, where $\mathcal F$ is the class of constant-valued functions, and $\mathcal G$ is the class of polynomial threshold functions of degree $\Delta'=O(\Delta)$ in $n$ variables (where $n$ is the number of parameters that specify an algorithm, as in \Cref{thm:GJ}). The VC-dimension of the dual class $\mathcal G^*$ can be upper-bounded by the number of monomials in $n$ variables of degree at most $\Delta'$, which is ${n+\Delta'\choose n}$.\footnote{The VC-dimension of the dual class could in principle be even smaller, if the primal class has a simpler structure than the dual class. However, in typical scenarios the primal class is less nicely structured than the dual class, which is the motivation for the work of \cite{balcan2021much}. Furthermore, if the primal class is indeed simpler, then there is no reason to go through duality at all.} Therefore, the main theorem of \cite{balcan2021much} implies an upper bound of $O\left({n+\Delta\choose n}\log\left({n+\Delta\choose n}\cdot p\right)\right)$ on the pseudo-dimension. Unfortunately, this is typically much looser than the bound $O(n\log(\Delta p))$ in \Cref{thm:GJ}. On the other hand, the result of \cite{balcan2021much} is much more general, and can handle decomposability beyond constant functions $\mathcal F$ and polynomial thresholds $\mathcal G$.

Specifically for IVY, the main component in our proof of \Cref{thm:main} was a GJ algorithm of degree $\Delta=O(mk(d/\epsilon)^{O(1/\epsilon)})$ and predicate complexity $p\leq2^m\cdot2^k\cdot(ed/k)^{3k}$ for computing the proxy loss (\Cref{clm:complexity}). 
While the main result of \cite{balcan2021much} can technically be applied here as described above, the pseudo-dimension upper bound it gives is super-polynomial in $n$ (as opposed to the linear dependence on $n$ in \Cref{thm:main}), so it does not seem to be useful in our setting.

\end{document}